%% file: arxiv_mk3.tex
\documentclass{article}

\input{prologue.tex}

\begin{document}

\title{Conditional Gradients for the Approximate Vanishing Ideal}

\author{\name Elias Wirth \email          \href{mailto:wirth@math.tu-berlin.de}{wirth@math.tu-berlin.de}\\
       \addr Institute of Mathematics,\\
        Berlin Institute of Technology, Berlin, Germany
       \AND
      \name Sebastian Pokutta \email \href{mailto:pokutta@zib.de}{pokutta@zib.de} \\
       \addr Institute of Mathematics,\\
       Berlin Institute of Technology, Berlin, Germany \\
        \& Zuse Institute Berlin, Berlin, Germany}

\maketitle

\begin{abstract}
  The vanishing ideal of a set of points $X\subseteq \R^n$ is the set of polynomials that evaluate to $0$ over all points $\xx \in X$ and admits an efficient representation by a finite set of polynomials called generators. To accommodate the noise in the data set, we introduce the \emph{pairwise conditional gradients approximate vanishing ideal algorithm} (\pcg{}\avi{}) that constructs a set of generators of the approximate vanishing ideal. The constructed generators capture polynomial structures in data and give rise to a feature map that can, for example, be used in combination with a linear classifier for supervised learning. In \pcg{}\avi{}, we construct the set of generators by solving constrained convex optimization problems with the \emph{pairwise conditional gradients algorithm}. Thus, \pcg{}\avi{} not only constructs few but also sparse generators, making the corresponding feature transformation robust and compact. Furthermore, we derive several learning guarantees for \pcg{}\avi{} that make the algorithm theoretically better motivated than related generator-constructing methods.
\end{abstract}

\section{Introduction}\label{section:introduction}
The accuracy of classification algorithms relies on the quality of the available features. Naturally, feature transformations are an important area of research in the field of machine learning. In this paper, we focus on feature transformations for a subsequently applied linear kernel \emph{support vector machine} (\svm{}) \citep{suykens1999least}, an algorithm that achieves high accuracy only if the features are such that the different classes are linearly separable. 
The approach is based on the idea that a set $X= \{\xx_1, \ldots, \xx_m\} \subseteq \R^n$ can be described by the set of algebraic equations satisfied by each point  $\xx \in X$. Put differently, we seek polynomials $g_1, \ldots, g_k\in\cP$ with $k\in \N$ such that $g_i(\xx) = 0$ for all $\xx \in X$ and $i\in \{1,\ldots, k\}$, where $\cP$ is the polynomial ring in $n$ variables.
An obvious candidate for a succinct description of $X$ is the \emph{vanishing ideal}\footnote{A set of polynomials $\cI \subseteq \cP$ is called an \emph{ideal} if it is a subgroup with respect to addition and for $f \in \cI$ and $g\in \cP$, we have $f \cdot g \in \cI$.} of $X$,
\begin{align*}
    \cI_X = \{f \in \cP \mid f(\xx) = 0 \text{ for all } \xx \in X\}.
\end{align*}
Even though $\cI_X$ contains an infinite number of vanishing polynomials, by Hilbert's basis theorem \citep{cox2013ideals}, there exists a finite number of \emph{generators} of $\cI_X$, $g_1, \ldots, g_k \in \cI_X$ with $k \in \N$, such that for any $f\in \cI_X$, there exist $h_1, \ldots, h_k \in \cP$ such that
$
    f = \sum_{i = 1}^k g_i h_i.
$
Generators have all points in $X$ as common roots, thus capture the nonlinear structure of the data set $X$, and methods for their construction have received a lot of attention \citep{heldt2009approximate, fassino2010almost, livni2013vanishing, limbeck2013computation, iraji2017principal}.
Since generators succinctly describe the data set $X$, they can, for example, be used in classification to map a non-separable data set $X$ into a higher-dimensional feature space in which the data becomes linearly separable. To illustrate this idea, consider the binary classification task of deciding whether a tumor is cancerous or benign and suppose that the sets $X^{c} \subseteq X$ and $X^{b} \subseteq X$ correspond to data points of cancerous or benign tumors, respectively. Then, generators $g_1, \ldots, g_k$ of $\cI_{X^{c}}$ provide a succinct characterization of elements in the class of cancerous tumors and vanish over $X^{c}$. For points $\xx \in X^{b}$, however, we expect that for some $i \in \{1, \ldots, k\}$, the polynomial $g_i$ does not vanish over $\xx$, that is, $g_i(\xx) \neq 0$. Similarly, we construct a second set of generators of $\cI_{X^{b}}$, representing the benign tumors. Then, evaluating both sets of generators over the entire data set $X = X^{c} \cup X^{b}$ and taking the absolute value represents the data set $X$ mapped into a new feature space in which the two classes are (ideally) linearly separable.
In practice, to accommodate the noise in the data set, we construct generators $g$ of the \emph{approximate vanishing ideal} instead of the vanishing ideal, where the approximate vanishing ideal contains polynomials that almost vanish over $X$.

\subsection{Contributions}
In this paper, we introduce the \emph{pairwise conditional gradients approximate vanishing ideal algorithm} (\pcg{}\avi{}), which takes as input a set of points $X\subseteq \R^n$ and constructs a set of generators of the approximate vanishing ideal over $X$. \pcg{}\avi{} repeatedly solves constrained convex optimization problems using an oracle, in our case, via the \emph{pairwise conditional gradients algorithm} (\pcg{}) \citep{guelat1986some,lacoste2015global},
whereas related methods such as the \emph{approximate Buchberger-Möller algorithm} (\abm{}) \citep{limbeck2013computation} and \emph{vanishing component analysis} (\vca{}) \citep{livni2013vanishing} rely on \emph{singular value decompositions} (\svd{}s) to construct generators. 
\pcg{}\avi{} admits the following properties:
\begin{enumerate}
\item \textbf{Generalization Bounds:} We derive generalization bounds for \pcg{}\avi{}, making \pcg{}\avi{} the first algorithm among related methods such as \abm{} and \vca{} with learning guarantees on out-sample data. Under mild assumptions, \pcg{}\avi{}'s generators are guaranteed to also vanish approximately on out-sample data and the combined approach of constructing generators with \pcg{}\avi{} to transform features for a subsequently applied linear kernel \svm{} satisfies a margin bound.

\item \textbf{Compact Transformation:} 
\pcg{}\avi{} constructs a small number of generators with sparse coefficient vectors.

\item \textbf{Blueprint:} For the implementation of the oracle in \pcg{}\avi{}, \pcg{} can be replaced by any other solver of (constrained) convex optimization problems. Thus, our paper gives rise to an entire family of procedures for the construction of generators of the approximate vanishing ideal.

\item \textbf{Empirical Results:}
For the combined approach of constructing generators to transform features for a linear kernel \svm{}, generators constructed with \pcg{}\avi{} are more sparse than and lead to test set classification errors and evaluation times comparable to related methods such as \abm{} and \vca{}.
\end{enumerate}

\subsection{Related Works}\label{section:related_works}

The first algorithm for constructing generators of the vanishing ideal was the \emph{Buchberger-Möller algorithm} \citep{moller1982construction}. Its high susceptibility to noise was first addressed with the \emph{approximate vanishing ideal algorithm} (\avi{}) \citep{heldt2009approximate}. Another algorithm constructing generators of the approximate vanishing ideal is \abm{} \citep{limbeck2013computation}, which offers more control on the extent of vanishing of the constructed generators, requires fewer subroutines, and is easier to implement than \avi{}. \citet{limbeck2013computation} also introduced the \emph{border bases Buchberger-Möller algorithm} (\bbabm{}), which constructs generators by repeatedly solving unconstrained convex optimization problems, similar to \pcg{}\avi{}, which constructs generators by repeatedly solving constrained convex optimization problems. \abm{}, \avi{}, \bbabm{}, and \pcg{}\avi{} are \emph{monomial-aware} algorithms. They require a term ordering and construct generators as linear combinations of monomials.
The term-ordering requirement is the reason why monomial-aware algorithms can produce different outputs depending on the order of the features of the data set, an undesirable property in practice.
The currently most prevalent, and contrary to the monomial-aware algorithms, \emph{monomial-agnostic} approach for constructing generators of the approximate vanishing ideal is \vca{}, introduced by \citet{livni2013vanishing} and improved by \citet{zhang2018improvement}. \vca{} constructs generators not as linear combinations of monomials but of polynomials, that is, generators constructed by \vca{} are, in some sense, polynomials whose terms are other polynomials. As a polynomial-based approach, \vca{} does not require an ordering of terms and the algorithm has been exploited in hand posture recognition, principal variety analysis for nonlinear data modeling, solution selection using genetic programming, and independent signal estimation for blind source separation tasks \citep{zhao2014hand,iraji2017principal,kera2016vanishing, wang2018nonlinear}.
Despite \vca{}'s prevalence, foregoing the term ordering of monomial-based approaches also gives rise to major disadvantages: \vca{} constructs more generators than monomial-aware algorithms, \vca{}'s generators are non-sparse in their polynomial representation, and \vca{} is highly susceptible to the \emph{spurious vanishing problem} \citep{kera2016vanishing, kera2019spurious, kera2020gradient}: Polynomials with small coefficient vector entries that vanish over $X$ still get added to the set of generators even though they do not hold any structurally useful information of the data, and, conversely, polynomials that describe the data well do not get recognized as approximately vanishing generators due to the size of their (large) coefficient vector entries.

Whereas \abm{}, \avi{}, \bbabm{}, and \vca{} construct generators using \svd{}s, \pcg{}\avi{} constructs generators using calls to \pcg{}.
\pcg{} is a variant of the \emph{Frank-Wolfe} \citep{frank1956algorithm} or \emph{conditional gradients} \citep{levitin1966constrained} algorithm (\cg{}). Conditional gradients methods are a family of algorithms that appear as building blocks in a variety of scenarios in machine learning, for example, structured prediction \citep{jaggi2010simple,giesen2012optimizing,harchaoui2012large,freund2017extended}, optimal transport \citep{courty2016optimal,paty2019subspace,luise2019sinkhorn}, and video co-localization \citep{joulin2014efficient,bojanowski2015weakly,peyre2017weakly}.
They have also been extensively studied theoretically, with various algorithmic variations \citep{garber2016linear,bashiri2017decomposition,braun2016lazifying,combettes2020boosting} and several accelerated convergence regimes~\cite{lacoste2013affine, garber2015faster, garber2016linearly}.
Furthermore, Frank-Wolfe algorithms enjoy many appealing properties \citep{jaggi2013revisiting}:
They are easy to implement, projection-free, do not require affine pre-conditioners \citep{kerdreux2021affine}, and variants, for example, \pcg{}, offer a simple trade-off between optimization accuracy and sparsity of iterates.
All these properties make them appealing algorithmic procedures for practitioners that work at scale.
Although Frank-Wolfe algorithms have been considered in polynomial regression, particle filtering, or as pruning methods of infinite RBMS \citep{blondel2017multi, bach2012equivalence, ping2016learning}, their favorable properties have not been exploited in the context of approximate vanishing ideals.

\subsection{Outline}

In Section~\ref{section:preliminaries}, we introduce background material. In Section~\ref{section:convex_optimization}, we reformulate the construction of generators as a convex optimization problem. In Section~\ref{section:the_oavi_algorithm}, we introduce the \emph{oracle approximate vanishing ideal algorithm} (\oavi{}), an algorithmic framework that captures \pcg{}\avi{}.
In Section~\ref{section:the_cgavi_pipeline}, we present the machine learning pipeline of using \oavi{} to transform features for a subsequently applied linear kernel \svm{}.
In Section~\ref{section:generalization_bounds}, we present \oavi{}'s generalization bounds.
In Section~\ref{section:conditional_gradients}, we discuss how conditional gradients can be used to construct generators.
In Section~\ref{sec:note_on_borders}, we discuss the effects of different borders on generator-constructing algorithms.
In Section~\ref{section:numerical_experiments}, we present empirical results. In Section~\ref{section: discussion}, we discuss our work.

\section{Preliminaries}\label{section:preliminaries}
Throughout, let $k,m,n \in \N$. We denote vectors in bold and let $\zeros \in \R^n$ denote the $0$-vector.
Throughout, we use capital calligraphic letters to denote sets of polynomials and denote the sets of terms (or monomials) and polynomials in $n$ variables by $\cT$ and $\cP$, respectively. 
We denote the constant-$1$ monomial by $\oneterm$.
Given a polynomial $f\in \cP$, let $\deg(f)$ denote its \emph{degree}.
We denote the sets of polynomials in $n$ variables of and up to degree $d$ by $\cP_{ d}$ and $\cP_{\leq d}$, respectively. Given a set of polynomials $\cG = \{g_1, \ldots, g_k\}\subseteq \cP \subseteq \cP$, let $\cG_d := \cG \cap \cP_d$ and $\cG_{\leq d}: = \cG \cap \cP_{\leq d}$.
Given a vector $\xx \in \R^n$, define the \emph{evaluation vector} of $\cG$ over $\xx$ as 
$
    \cG(\xx):=(g_1(\xx), \ldots, g_k(\xx))^\intercal\in \R^k.
$
Throughout, let $X = \{\xx_1, \ldots, \xx_m\} \subseteq \R^n$ be a data set consisting of $m$ $n$-dimensional feature vectors. 
Given a polynomial  $f\in \cP$ and a set of polynomials $\cG= \{g_1, \ldots, g_k\}\subseteq \cP$, define the \emph{evaluation vector} of $f$ and the \emph{evaluation matrix} of $\cG$ over $X$ as
$f(X) : = (f(\xx_1), \ldots, f(\xx_m) )^\intercal\in \R^m$ and $\cG(X):=(g_1(X), \ldots, g_k(X))\in \R^{m \times k}$,
respectively. Further, define the \emph{mean squared error} of $f$ over $X$ as
$
    \mse(f,X) :=\frac{1}{|X|} \left\|f(X)\right\|_2^2 = \frac{1}{m} \left\|f(X)\right\|_2^2 , 
$
often referred to as the \emph{extent of vanishing} of $f$ over $X$.
Below, we define approximately vanishing polynomials using the mean squared error.
\begin{definition}
[Approximately vanishing polynomial]
\label{definition:approximately_vanishing_polynomial}
Let $X=\{\xx_1, \ldots, \xx_m\}\subseteq \R^n$ and $\psi \geq 0$. A polynomial $f\in\cP$ is \emph{$\psi$-approximately vanishing} (over $X$) if $\mse(f,X) \leq \psi$.
\end{definition}
Recall the spurious vanishing problem:
For $X=\{\xx_1, \ldots, \xx_m\}\subseteq \R^n$, any polynomial $f$ with $\mse(f, X) > \psi$ can be re-scaled to become $\psi$-approximately vanishing, regardless of its roots, by multiplying all coefficient vector entries of $f$ with $\sqrt{\psi / \mse(f, X)}$. To address the issue, we require an ordering of terms, in our case, the \emph{degree-lexicographical ordering of terms} (DegLex) \citep{cox2013ideals}, denoted by $<_\sig$. For example, for $t_1, t_2 \in \cT$, we have $\oneterm <_\sig t_1 <_\sig t_2 <_\sig t_1^2 <_\sig t_1\cdot t_2 <_\sig t_2^2\ldots$, where $\oneterm$ denotes the constant-$1$ monomial.
Throughout, for $\cO = \{t_1, \ldots, t_{k}\}_\sig \subseteq \cT$, the subscript $\sig$ indicates that $t_1 <_\sig \ldots <_\sig t_{k}$.
\begin{definition}[Leading term (coefficient)]\label{def:ltc}
Let $f = \sum_{i=1}^kc_i t_i\in \cP$, where $k\in \N$ and $c_i \in \R$ and $t_i \in \cT$ for all $i \in \{1, \ldots, k\}$ and let $j\in \{1, \ldots, k\}$ such that $t_j >_\sig t_i$ for all $i \in\{1,\ldots,k\}\setminus\{j\}$.
Then, $t_j$ and $c_j$ are called \emph{leading term} and \emph{leading term coefficient} of $f$, denoted by $\lt(f) = t_j$ and $\ltc(f) = c_j$, respectively.
\end{definition}

For $X=\{\xx_1, \ldots, \xx_m\}\subseteq \R^n$, a polynomial $f$ with $\ltc(f) = 1$ that vanishes $\psi$-approximately over $X$ is called \emph{$(\psi, 1)$-approximately vanishing} (over $X$).
Thus, fixing the leading term coefficient of generators prevents rescaling and addresses the spurious vanishing problem. We formally define the approximate vanishing ideal.

\begin{definition}
[Approximate vanishing ideal]
\label{definition:approximately_vanishing_ideal}
Let $X=\{\xx_1, \ldots, \xx_m\}\subseteq \R^n$ and $\psi \geq 0$.  The \emph{$\psi$-approximate vanishing ideal} (over $X$), $\cI^\psi_X$, is the ideal generated by all $(\psi, 1)$-approximately vanishing polynomials over $X$.
\end{definition}
Note that the definition above subsumes the definition of the vanishing ideal, that is, for $\psi = 0$, $\cI^0_X =\cI_X $ is the vanishing ideal.
In this paper, we introduce an algorithm addressing the following problem.
\begin{problem}
[Setting]
\label{problem:oavi}
Let $X=\{\xx_1, \ldots, \xx_m\}\subseteq \R^n$ and $\psi\geq 0$.
Construct a set of $(\psi, 1)$-approximately vanishing generators of $\cI^\psi_X$.
\end{problem}

\section{Convex Optimization}\label{section:convex_optimization}

To address Problem~\ref{problem:oavi}, we first consider the subproblem of certifying (non-)existence of and constructing generators of the $\psi$-approximate vanishing ideal, $\cI^\psi_X$, when terms of generators are contained in a specific set of terms. As we show below, this subproblem can be reduced to solving a convex optimization problem. 

Let $X=\{\xx_1, \ldots, \xx_m\}\subseteq \R^n$,  $\psi \geq 0$, $\cO = \{t_1, \ldots,t_k\}_\sig \subseteq \cT$, and $t\in \cT$ such that $t >_\sig t_i$ for all $i\in \{1, \ldots, k\}$. 
Suppose there exists a $(\psi, 1)$-approximately vanishing polynomial $f = \sum_{i = 1}^{k} c_i t_i +  t$, where $k\in \N$ and $c_i \in \R$ and $t_i \in \cT$ for all $i \in \{1, \ldots, k\}$, with $\lt(f) = t$ and non-leading terms only in $\cO$. Let $\cc= (c_1,\ldots,c_k)^\intercal$.
Then,
\begin{align}\label{eq:equation_rhs_cop}
    \psi  \geq  \mse(f, X) & =  \ell( \cO(X),   t(X))(\cc) \geq \min_{\vv \in \R^k}  \ell( \cO(X),   t(X))(\vv),
\end{align}
where for $A \in \R^{m \times k}$, $\bb \in \R^m$, and $\xx \in \R^k$, the \emph{squared loss} is defined as
$
    \ell(A, \bb)(\xx):= \frac{1}{m} \|A \xx + \bb\|_2^2.
$
The right-hand side of \eqref{eq:equation_rhs_cop} is a convex optimization problem and in the $\argmin$-version has the form
\begin{equation}\tag{COP}\label{eq:cop}
        \dd \in \argmin_{\vv \in \R^k} \ell(\cO(X),   t(X))(\vv).
\end{equation}
Thus, $g =  \sum_{i = 1}^{k} d_{i} t_i + t$ vanishes $(\psi, 1)$-approximately,  $\lt(g)=t$, and non-leading terms of $g$ are in $\cO$.
We obtain the following result.
\begin{theorem}
[Certificate]
\label{thm:certificate_polynomial}
Let $X = \{\xx_1, \ldots, \xx_m\} \subseteq \R^n$, $\psi\geq 0$, $\cO = \{t_1, \ldots t_k\}_\sig \subseteq \cT$, $t\in \cT$ such that $t >_\sig t_i$ for all $i\in \{1, \ldots, k\}$, and $\dd$ as in \eqref{eq:cop}.
There exists a $(\psi,1)$-approximately vanishing polynomial $f$ with $\lt(f) = t$ and non-leading terms only in $\cO$, if and only if $g =  \sum_{i = 1}^{k} d_{i} t_i +  t$ is $(\psi, 1)$-approximately vanishing.
\end{theorem}

\begin{algorithm}[t]
\SetKwInOut{Input}{Input}\SetKwInOut{Output}{Output}
\SetKwComment{Comment}{$\triangleright$\ }{}
\Input{$X = \{\xx_1, \ldots, \xx_m\}\subseteq \R^n$, $\cO = \{t_1, \ldots, t_k\}_\sig \subseteq \cT$, $t\in\cT$ with $t >_\sig t_k$, and $\eps \geq 0$.}
\Output{A polynomial $g\in \cP$ with $\lt(g) = t$, $\ltc(g) = 1$, non-leading terms only in $\cO$, and $\mse(g, X) \leq \min_{\vv \in \R^{k}}\ell(\cO(X), t(X))(\vv) + \eps$.}
\hrulealg
\caption{\cvxoracle{}}
 \label{algorithm:convex_oracle}
\end{algorithm}

The discussion above not only constitutes a proof of Theorem~\ref{thm:certificate_polynomial} but also gives rise to an algorithmic blueprint, \cvxoracle{}, presented in Algorithm~\ref{algorithm:convex_oracle}. 
In practice, any $\eps$-accurate solver of \eqref{eq:cop}, for example, gradient descent, can be used to implement \cvxoracle{} in two steps: First, solve problem \eqref{eq:cop} to $\eps$-accuracy with the solver yielding a vector $\dd\in \R^{k}$. Second, construct and return $g =  \sum_{i = 1}^{k} d_{i} t_i + t$. 
In case \cvxoracle{} is implemented with an accurate solver of \eqref{eq:cop}, that is, $\eps = 0$, by Theorem~\ref{thm:certificate_polynomial}, either $\mse(g,X) > \psi$ and we have proof that no $(\psi, 1)$-approximately vanishing polynomial with leading term $t$ and non-leading terms only in $\cO$ exists, or $g$ is $(\psi, 1)$-approximately vanishing with leading term $t$ and non-leading terms only in $\cO$. 
We denote the output of running \cvxoracle{} with $X, \cO, t$, and $\eps$ by $g=\cvxoracle{} (X, \cO, t, \eps)$. 

\section{The Oracle Approximate Vanishing Ideal Algorithm (\oavi{})}\label{section:the_oavi_algorithm}

We introduce and study the oracle approximate vanishing ideal algorithm (\oavi{}) in Algorithm~\ref{algorithm:oavi}, an algorithmic framework that captures \pcg{}\avi{} and addresses Problem~\ref{problem:oavi}.

\subsection{Algorithm Overview}

We present an overview of \oavi{}, which we refer to as \pcg{}\avi{} when \cvxoracle{} is implemented with \pcg{}.
\subsubsection{Input}
Recall Problem~\ref{problem:oavi}. For $X=\{\xx_1, \ldots, \xx_m\}\subseteq \R^n$, \oavi{} constructs generators of the $\psi$-approximate vanishing ideal, $\cI_X^\psi$, where $\psi$ controls the extent of vanishing of generators via calls to \cvxoracle{}, the accuracy of which is controlled by the tolerance $\eps$.
For ease of presentation, $\psi = \eps = 0$, that is, we focus on constructing generators of the vanishing ideal $\cI_X$ with an accurate solver of \eqref{eq:cop}.

\subsubsection{Initialization} 
We keep track of two sets: $\cO \subseteq \cT$ for the set of terms such that no generator exists with terms only in $\cO$ and $\cG\subseteq \cP$ the set of generators.
Since the constant-$1$ polynomial does not vanish, we initialize $\cO=\{t_1\}_\sig \gets \{\oneterm\}_\sig$ and $\cG \gets \emptyset$, where $\oneterm$ denotes the constant-$1$ monomial.

\subsubsection{Line~\ref{alg:while_loop}}
Given $\cO_{\leq d-1}$ and $\cG_{\leq d-1}$, \oavi{} determines whether generators of degree $d$ with non-leading terms only in $\cO_{\leq d- 1}$ exist.
Checking for $\binom{n}{d}$ monomials of degree $d$ whether they are the leading term of a generator is impractical. Recall the following result, which states that there exists a set of generators $\cG$ of $\cI_X$ such 
that none of the terms of generators in $\cG$ are divisible\footnote{Recall that for $t, u \in \cT$, $t$ \emph{divides (or is a divisor of)} $u$, denoted $t\mid u$, if there exists $v\in \cT$ such that $t \cdot  v = u$. If $t$ does not divide $u$, we write $t\nmid u$.} by leading terms of other generators in $\cG$.
\begin{lemma}[{\citealp[Theorem~2.4.12]{kreuzer2000computational}}]
\label{lemma:purging_property}
Let $X=\{\xx_1, \ldots, \xx_m\}\subseteq \R^n$. There exists a set of generators $\cG\subseteq\cP$ of $\cI_X$ such 
that for $g, h \in \cG$ with $g\neq h$ and $h =\sum_{i=1}^k c_i t_i\in \cG$, where $k\in \N$ and $c_i \in \R$ and $t_i \in \cT$ for all $i \in \{1, \ldots, k\}$, it holds that $\lt(g) \nmid t_i$ for any $i \in \{1, \ldots, k\}$.
\end{lemma}
To construct degree-$d$ generators, we thus only consider terms $t \in \cT_d$ such that for all $g \in \cG_{\leq d-1}$, it holds that $\lt(g) \nmid t$. This is equivalent to requiring that all divisors of degree $\leq d-1$ of $t$ are in $\cO_{\leq d-1}$. 
\begin{definition}
[Border]
\label{definition:border}
Let $\cO\subseteq \cT$. The \emph{(degree-$d$) border} of $\cO$ is defined as
\begin{align}\label{eq:gb}\tag{\grba{}}
    \partial_d \cO:= \{u \in \cT_{d} \colon t \in \cO_{\leq d-1} \text{ for all } t \in \cT_{\leq d-1} \text{ such that } t \mid u\}.
\end{align}
\end{definition}

In other words, we only consider degree-$d$ terms that are contained in the border, which drastically reduces the number of redundant generators constructed compared to a naive approach, see also Section~\ref{sec:note_on_borders}.

\subsubsection{While-Loop}
Suppose that $\partial_d \cO \neq \emptyset$, else \oavi{} terminates.
For each term $u\in\partial_d\cO$, starting with the smallest with respect to $<_\sig$, we construct a polynomial $g$ via a call to \cvxoracle{}. By Theorem~\ref{thm:certificate_polynomial}, there exists a vanishing polynomial $f$ with $\ltc(f)=1$, $\lt(f) = u$, and non-leading terms only in $\cO$, if and only if $g$ is a vanishing polynomial.
If $g$ vanishes, we append $g$ to $\cG$. If $g$ does not vanish, we append $\lt(g) = u$ to $\cO$.

\subsubsection{Termination} When the border is empty, \oavi{} terminates with outputs $\cG$ and $\cO$, that is, $(\cG, \cO) = \oavi{}(X,\psi,\eps)$.

\begin{algorithm}[t]
\SetKwInOut{Input}{Input}\SetKwInOut{Output}{Output}
\SetKwComment{Comment}{$\triangleright$\ }{}
\Input{A data set $X = \{\xx_1, \ldots, \xx_m\} \subseteq\R^n$ and parameters $1 > \psi\geq \eps \geq 0$.}
\Output{A set of polynomials $\cG\subseteq \cP$ and a set of monomials $\cO\subseteq \cT$.}
\hrulealg
{$d \gets 1$} \\
{$\cO=\{t_1\}_\sig \gets \{\oneterm\}_\sig$ \Comment*[f]{the monomial $\oneterm$ is the constant-$1$ monomial}}\\
{$\cG\gets \emptyset$}\\
\While(\label{alg:while_loop}\Comment*[f]{repeat for as long as the border is non-empty}){$\partial_d \cO = \{u_1, \ldots, u_k\}_\sig \neq \emptyset$}{
    \For(\label{alg:for_loop}){$i = 1, \ldots, k$}{
        {$g \gets \cvxoracle{} (X, \cO, u_i, \eps) \in \cP$}\label{alg:cvx_oracle}\\
        \uIf (\Comment*[f]{determine the extent of vanishing of $g$}){$\mse(g, X) \leq \psi$}{
            {$\cG \gets \cG \cup \{g\}$\label{alg:g}}
            } 
        \Else {
        {$\cO\gets (\cO \cup \{u_i\})_{\sig}$\label{alg:o}}
        }
    }
    {$ d \gets d + 1$} 
}
\caption{Oracle Approximate Vanishing Ideal Algorithm (\oavi{})} \label{algorithm:oavi}
\end{algorithm}

\subsection{Analysis}\label{section:analysis}

For the remainder of this section, we analyze \oavi{} and properties of the algorithm's output.
First, we prove that \oavi{} terminates with $|\cO| \leq m$ and $|\cG| \leq |\cO| n$.

\begin{proposition}\label{prop:boundedness_of_O}
Let $X =\{\xx_1, \ldots, \xx_m\} \subseteq \R^n$, $1 > \psi \geq  \eps \geq 0$, and $(\cG, \cO) = \oavi{}(X,\psi,\eps)$. Then, $|\cO| \leq m$ and $|\cG|\leq |\cO|n$.
\end{proposition}
\begin{proof}
Suppose that at some point during \oavi{}'s execution, $|\cO| = m$. Then, $\cO(X)\in \R^{m\times m}$, and for any term $u$ in the current or upcoming border, $u(X)$ can be written as a linear combination of columns of $\cO(X)$, a vanishing polynomial with leading term $u$ and non-leading terms in $\cO$ is detected, and no more terms get appended to $\cO$.
Let $D\in \N$ denote the degree after which \oavi{} terminates. By construction, leading terms of polynomials in $\cG$ are contained in $\bigcup_{d = 1}^D\partial_d\cO$. Hence, $|\cG|\leq \left|\bigcup_{d = 1}^D\partial_d\cO \right|\leq |\cO|n$.
\end{proof}

\subsubsection{Generator Construction}

We next focus on characterizing how \oavi{} addresses Problem~\ref{problem:oavi}.
In practice, we employ solvers that are $\eps$-accurate for $\eps > 0$, in which case \oavi{} addresses Problem~\ref{problem:oavi} up to tolerance $\eps$.

\begin{theorem}
[Maximality]
\label{thm:maximality}
Let $X = \{\xx_1, \ldots, \xx_m\}  \subseteq \R^n$, $1 > \psi \geq  \eps \geq 0$, and $(\cG, \cO) = \oavi{}(X,\psi,\eps)$. Then, all $g \in \cG$ are $(\psi, 1)$-approximately vanishing and there does not exist a $(\psi- \eps, 1)$-approximately vanishing polynomial with terms only in $\cO$.
\end{theorem}
\begin{proof}
By construction, $\cG$ is a set consisting of $(\psi, 1)$-approximately vanishing polynomials.
Suppose towards a contradiction that there exists a $(\psi - \eps, 1)$-approximately vanishing polynomial $f $ with terms only in $\cO$ and $\lt(f) = t \in \cO$. Let $\cU:= \{u \in \cO \mid t >_\sig u\} = \{u_1, \ldots, u_k\}_\sig$. At some point during its execution, \oavi{} constructs a polynomial $g$ with $\lt(g) = t$ and non-leading terms only in $\cU$. 
By $\eps$-accuracy of \cvxoracle{},
$
    \mse(g, X) \leq \min_{\vv\in \R^k}  \frac{1}{m} \left\| \cU(X) \vv + t(X) \right\|_2^2 + \eps \leq \mse(f, X) + \eps  \leq \psi - \eps + \eps \leq \psi,
$
a contradiction.
\end{proof}

Thus, \oavi{} is guaranteed to construct all $(\psi - \eps, 1)$-approximately vanishing polynomials but generators that are $(\lambda, 1)$-approximately vanishing, where $\lambda \in ]\psi - \eps, \psi]$, may not be detected by \oavi{}. In case $\psi = \eps = 0$, \oavi{} constructs a set of generators of the vanishing ideal $\cI_X$, completely addressing Problem~\ref{problem:oavi}. 
The result below is similar to \citet[Theorem~5.2]{livni2013vanishing} but obtained with a different proof technique.

\begin{theorem}[Generating set]\label{theorem:G_generatres_I}
Let $X = \{\xx_1, \ldots, \xx_m\} \subseteq \R^n$, let $\psi =  \eps = 0$, and $(\cG, \cO) = \oavi{}(X,\psi,\eps)$. Then,
$\langle\cG\rangle = \cI_X$ and any $f \in \cP$ can be written as $f = g + h $, where $g\in \langle \cG \rangle$ and 
$h \in \mathspan(\cO)$.
\end{theorem}

\begin{lemma}\label{lemma:decomposition}
Let $X = \{\xx_1, \ldots, \xx_m\} \subseteq \R^n$, let $\psi =  \eps = 0$, $(\cG, \cO) = \oavi{}(X,\psi,\eps)$, and $d\in \N$. Then, $f \in \cP_{\leq d}$ can be written as $f = g + h $, where $g\in \langle \cG_{\leq d} \rangle$ and 
$h \in \mathspan(\cO_{\leq d})$.
\end{lemma}

\begin{proof}
The proof is by induction.
We say that $f\in \cP_{\leq d}$ can be \emph{decomposed} if there exist $g\in \langle \cG_{\leq d} \rangle$ and 
$h \in \mathspan(\cO_{\leq d})$ such that $f = g + h$. Polynomials of degree $1$ are decomposable.
Let $d\in\N$. Suppose that for $i < d$, any polynomial $f\in\cP_{\leq i}$ can be decomposed but that there exist non-decomposable polynomials of degree $d$. Let $f\in\cP_d$ be the non-decomposable polynomial with the degree-lexicographically smallest leading term. If $\lt(f) \in \cO_{d}$, let $p_1 = f - \ltc(f)\cdot\lt(f)\in \cP_{\leq d}$. Since $\lt(p_1) <_\sig \lt(f)$, $p_1$ is decomposable. As sum of decomposable polynomials, $f = p_1 + \ltc(f)\cdot\lt(f)$ is decomposable, a contradiction. If $\lt(f) \not\in \cO_{d}$, we distinguish between two cases. Either $f\in\langle \cG_{\leq d}\rangle$ and $f$ is decomposable, a contradiction, or $f\not\in\langle \cG_{\leq d}\rangle$. In the latter case, since $\lt(f)\not\in \cO_d$ and $f\not\in \cG_{\leq d}$, by Definition~\ref{definition:border}, it holds that $\lt(f)\not\in \partial_d\cO$. Thus, there has to exist a polynomial $p_2 \in \cG_{\leq d-1}$ such that $\lt(p_2) \mid \lt(f)$. Thus, there exists a polynomial $p_3 \in\cP_{\leq d-1}$ such that $\ltc(p_2\cdot p_3)\cdot\lt(p_2\cdot p_3) = \ltc(f)\cdot\lt(f)$. Let $p_4= f - p_2\cdot p_3\in \cP_{\leq d}$. Since $\lt(p_4) <_\sig \lt(f)$,
$p_4$ is decomposable. Since $p_2\in \cG_{\leq d-1}$, it holds that $p_2\cdot p_3\in \langle\cG_{\leq d}\rangle$. As sum of decomposable polynomials, $f = p_4 + p_2 \cdot p_3$ is decomposable, a contradiction.
\end{proof}

\begin{proof}[Proof of Theorem~\ref{theorem:G_generatres_I}]
Consider a polynomial $f\in \cI_X$. By Lemma~\ref{lemma:decomposition}, $f = g + h$, where $g \in \langle\cG\rangle$ and $h \in \mathspan(\cO)$. Thus, $\zeros = f(X) = g(X) + h(X) = h(X)$. By Theorem~\ref{thm:maximality}, $h = \zeros$. Thus, $\langle\cG\rangle = \cI_X$.
\end{proof}

\subsubsection{Computational Complexity}\label{section:computational_complexity}

We present \oavi{}'s time, space, and evaluation complexities below.

\begin{theorem}
[Time and space]
\label{thm:time_and_space}
Let $X = \{\xx_1, \ldots, \xx_m\} \subseteq \R^n$, $1 >\psi \geq \eps \geq 0$, and $(\cG, \cO) = \oavi{}(X,\psi,\eps)$. In the real number model, the time and space complexities of \oavi{} are $O((|\cG|+|\cO|)^2 + (|\cG|+|\cO|)T_\cvxoracle{})$ and $O((|\cG| + |\cO|)m + S_\cvxoracle{})$, where $T_\cvxoracle{}$ and $S_\cvxoracle{}$ are the time and space complexities of \cvxoracle{}, respectively.
\end{theorem}
\begin{proof}
Let $D\in \N$ denote the degree after which \oavi{} terminates. Let $d\in \{1, \ldots, D\}$ and suppose that $\cO_0, \ldots, \cO_{d-1}$ and $\partial_{1}\cO, \ldots, \partial_{d-1}\cO$ and the evaluation matrices $\cO_0(X), \ldots, \cO_{d-1}(X)$ and $\partial_{1}\cO(X), \ldots, \partial_{d-1}\cO(X)$ are already stored. Further, suppose that for all $g = \sum_ic_i t_i\in \cG_{\leq d-1}$, the coefficient vector $\cc$ and evaluation vector $g(X)$ are already stored. To execute Line~\ref{alg:while_loop} of \oavi{}, the algorithm constructs $\partial_d\cO$ and $\partial_d\cO(X)$. To construct the former, \oavi{} constructs
$\cC_d = \{u  = v \cdot t \in \cT_d \mid v\in \cO_1, t \in \cO_{d-1}\}$, requiring time $O(|\cO_{d-1}||\cO_1|)$. Then, to obtain $\partial_d\cO$, \oavi{} removes any $t\in \cT$ from $\cC_d$ such that $t = \lt(g)$ for some $g\in \cG_{\leq d - 1}$, requiring time $O(|\cC_d||\cG_{\leq d- 1}|) \leq O(|\cO_{d-1}||\cO_1||\cG|)$. Then, constructing $\partial_d\cO(X)$ consists of at most $|\cC_d|$ entry-wise multiplications of two $m$-dimensional evaluation vectors of terms in $\cO_{\leq d-1}$, requiring time $O(|\cC_d| m) = O(|\cO_{d-1}||\cO_1|m)$.
Thus, to construct $\bigcup_{d=1}^D\partial_d \cO$ and the evaluation vectors of terms therein, it requires time $ O(\sum_{d = 1}^D|\cO_{d-1}||\cO_1|(|\cG|+m)) = O((|\cG|+|\cO|)(|\cG|+ m)) \leq O((|\cG|+|\cO|)^2 + (|\cG|+|\cO|)m)$.
The interior of the for-loop in Line~\ref{alg:for_loop} gets executed once for each border term, that is, $O(\vert\bigcup_{d=1}^D\partial_d\cO\vert)=O(|\cG| + |\cO|)$ times. Without loss of generality, we assume that the cost of Line~\ref{alg:cvx_oracle} dominates the cost of Lines~\ref{alg:g} and~\ref{alg:o}, that is, $O(|\cO|m) \leq O(T_\cvxoracle{})$. Thus, the total time complexity of \oavi{} is $O((|\cG|+|\cO|)^2 + (|\cG|+|\cO|)m   + (|\cG|+|\cO|) T_\cvxoracle{}) = O((|\cG|+|\cO|)^2 + (|\cG|+|\cO|)T_\cvxoracle{})$.
Throughout \oavi{}'s execution, $O(|\cG| + |\cO| + \sum_{d=1}^D|\cC_d|) = O(|\cG| + |\cO|)$ terms and corresponding evaluation vectors are stored. Further, for all $g = \sum_ic_i t_i \in \cG$, the coefficient vector $\cc$ and evaluation vector $g(X)$ are stored. Since terms are stored in $O(1)$ and coefficient and evaluation vectors are stored in $O(m)$, the total space complexity of \oavi{} is $O((|\cG| + |\cO|)m + S_\cvxoracle{})$.
\end{proof}

\begin{theorem}
[Evaluation]
\label{thm:evaluation_of_O_and_G}
Let $X = \{\xx_1, \ldots, \xx_m\}  \subseteq \R^n$, $1 > \psi \geq \eps \geq 0$, and $(\cG, \cO) = \oavi{}(X,\psi,\eps)$. In the real number model, the evaluation vectors of all monomials in $\cO$ and polynomials in $\cG$ over a set $Z=\{\zz_1, \ldots, \zz_q\}  \subseteq \R^n$ can be computed in times $O( |\cO|q)$ and $O(|\cG||\cO| q )$, respectively.
\end{theorem}

\begin{proof}
The proof is an adaptation of the proof of \citet[Theorem~5.1]{livni2013vanishing} to \oavi{}. Let $\cO = \{t_1, \ldots, t_k\}_\sig$ with $k \leq m$. Since $t_1 = \oneterm$, $t_1(X)$ can be computed in time $O(q)$. For $d\in \N_{>0}$, the evaluation vector of terms $t\in \cO_d$ over $Z$ can be computed by multiplying the evaluation vectors of two terms in $\cO_{\leq d-1}$ element-wise, requiring time $O(q)$. Hence, the evaluation vectors of all monomials in $\cO$ over $Z$ can be constructed in time $O(|\cO|q )$. Since the evaluation vectors of leading terms of generators in $\cG$ are element-wise multiplications of evaluation vectors in $\cO$, we can construct all leading terms of generators in $\cG$ in time $O((|\cG| + |\cO|)q)$.
The evaluation vectors of generators in $\cG$ are linear combinations of at most $|\cO| + 1$ evaluation vectors of terms. The computation of the linear combinations requires time $O(|\cG||\cO|q)$.
\end{proof}

Thus, the computational complexity of \oavi{} benefits from constructing fewer terms in $\cO$ and generators in $\cG$, increasing the sparsity of generators in $\cG$, and improving the computational complexity of \cvxoracle{}.

\subsubsection{Order Ideals: Exploiting Structure in \oavi{}'s Output}\label{sec:order_ideals}

To derive generalization bounds in Section~\ref{section:generalization_bounds}, we further study the structure of \oavi{}'s output for $X \subseteq \R^n$ and $1 > \psi \geq \eps \geq 0$, that is, $(\cG, \cO) = \oavi{}(X,\psi,\eps)$, which, by Definition~\ref{definition:border}, forms an \emph{order ideal}.

\begin{definition}[Order ideal]\label{def:order_ideal}
A set $\cO\subseteq \cT$ is an \emph{order ideal} if for all $t\in \cO$ and $u\in \cT$ such that $u\mid t$, we have $u\in\cO$. Let $\Xi$ denote the set of all order ideals. For $k\in \N$, let $\Xi_k := \{\cU \in \Xi \colon |\cU| = k\}$.
\end{definition}

\begin{lemma}[Order ideal]\label{lemma:order_ideal}
Let $X = \{\xx_1, \ldots, \xx_m\} \subseteq \R^n$, $1 >\psi \geq \eps \geq 0$, and $(\cG, \cO) = \oavi{}(X,\psi,\eps)$. Then, $\cO$ and
$
    \cO \cup \{ t \in \cT \mid t = \lt(g) \ \text{for some} \ g \in \cG\}
$
are order ideals.
\end{lemma}

Below, we provide a coarse bound on the number of order ideals of size $k$.

\begin{lemma}[Coarse upper bound on $|\Xi_{k}|$]\label{lemma:coarse_Xi}
For $k\in \N$, it holds that $|\Xi_{k}| \leq n^k k ! \leq (n k)^k$.
\end{lemma}
\begin{proof}
By induction, we prove $|\Xi_k| \leq n^{k} k !$ for all $k \in \N$. Since $\{\{\oneterm\}\} = \Xi_1$, the claim holds for $k = 1$.
Suppose that $|\Xi_i| \leq n^{i} i !$ for all $i\leq k \in \N$. 
For any $\cU' = \{t_1, \ldots, t_{k+1}\}_\sig \in \Xi_{k+1}$, removing the degree-lexicographically largest term $t_{k+1}$ yields an order ideal $\cU = \{t_1, \ldots, t_k\}_\sig\in\Xi_{k}$ such that $t_{k+1} = u \cdot v$ for some $u \in \cU_{\deg(t_{k+1}-1)}\subseteq \cU$ and $v\in \cU_1\subseteq \cT_1$.
Thus,
$
    \Xi_{k+1} \subseteq \bigcup_{\cU\in \Xi_k} \bigcup_{\substack{t = u \cdot v\\
u\in \cU, v\in \cT_1}}(\cU \cup \{t\})
$
and 
$
    |\Xi_{k+1}| \leq |\Xi_{k}| \cdot k \cdot |\cT_1|   =  n^k (k!) \cdot  k\cdot n    \leq  n^{k+1} ((k+1)!).
$
\end{proof}

\section{Pipeline}\label{section:the_cgavi_pipeline}

We present a detailed overview of the machine learning pipeline for classification problems using generators of the vanishing ideal to transform features for a linear kernel \svm{} with $\ell_1$-bounded weight vector,
following the notation of \citet{mohri2018foundations}. For ease of exposition, we assume that $\psi = \epsilon = 0$.
Consider an input space $\cX \subseteq \R^n$ and an output or target space $\cY = \{1, \ldots, k\}$. We receive a training sample $X = (\xx_1, \ldots, \xx_m)\in \cX^m$ drawn i.i.d. according to a distribution $\cD$.\footnote{For the theoretical foundation of \oavi{}, we adopt $X \subseteq \R^n$ as a set, in accordance with traditional algebraic geometry definitions. The extension of this analysis to accommodate $X$ as a sample $X \in \mathcal{X}^m$ is straight-forward. This sample notation, which allows for duplicates, is crucial for the statistical learning aspects of our work, particularly for deriving learning guarantees. Thus, our approach effectively marries the rigorous framework of algebraic geometry with the practical necessities of statistical learning.} Given a target function $f\colon \cX \to\cY$, the problem is to determine a \emph{hypothesis} $h\colon \cX \to \cY$ with small \emph{generalization error}
$
    \P_{\xx \backsim \cD}[h(\xx) \neq f(\xx)].
$
For each class $i\in\{1, \ldots, k\}$, let $X^i \subseteq X$ denote the subsample of feature vectors corresponding to class $i$ and construct a set of generators $\cG^i$ for the vanishing ideal $\cI_{X^i}$. Let $\cG:= \bigcup_{i=1}^k \cG^i = \{g_1, \ldots, g_{|\cG|}\}$.
We then transform samples $\xx\in X$ via the feature transformation
\begin{equation}\label{eq:feature_transformation}\tag{FT}
    \xx \mapsto \tilde{\xx} =  \left(|g_1(\xx)|, \ldots, |g_{|\cG|} (\xx)| \right)^\intercal.
\end{equation}
A polynomial $g\in \cG^i$ vanishes over all $\xx \in X^i$ and (ideally) attains non-zero values over points $\xx \in X \setminus X^i$. 
We then train a linear kernel \svm{} on the feature-transformed data $\tilde{X} = (\tilde{\xx}_1, , \ldots, \tilde{\xx}_m)$ with modified target function $\tilde{f}\colon \tilde{\xx} \mapsto f(\xx)$ with $\ell_1$-regularization to keep the number of used features as small as possible. If the underlying classes of $X$ belong to disjoint \emph{algebraic sets}\footnote{A set $U \in \R^n$ is \emph{algebraic} if there exists a finite set of polynomials $\cU\subseteq \cP$, such that $U$ is the set of the common roots of $\cU$.}, they become linearly separable in the feature space corresponding to transformation \eqref{eq:feature_transformation}, and perfect classification accuracy can be achieved on the training set \citep{livni2013vanishing}.

\section{Generalization Bounds}\label{section:generalization_bounds}

In this section, we present modifications to \eqref{eq:cop} that allow \oavi{} to satisfy several generalization bounds.
For $\tau \geq 2$, a polynomial $f = \sum_{i=1}^k c_i t_i\in \cP$ with $\cc = (c_1, \ldots, c_k)^\intercal\in \R^k$ and $t_1, \ldots, t_k \in \cT$ is said to be \emph{$\tau$-bounded} in norm $\|\cdot\|$ if the norm of its coefficient vector is bounded by $\tau$, that is, if $\|f\| := \|\cc\| \leq \tau$. Replacing \eqref{eq:cop} in \oavi{} by 
\begin{equation}\tag{CCOP}\label{eq:ccop}
        \dd \in \argmin_{\vv \in \R^k, \|\vv\| \leq \tau - 1} \ell(\cO(X),   t(X))(\vv)
\end{equation}
allows \oavi{} to create $\tau$-bounded generators in $\|\cdot\|$. 
Under mild assumptions, we demonstrate that \oavi{} run with \eqref{eq:ccop} for the norm $\|\cdot\|_1$ and $\tau \geq 2$ admits several learning guarantees, relying on the fact that the constructed generators have coefficient vectors that are bounded in the $\ell_1$-norm.

As a gentle introduction, consider a data set $X = \{\xx_1, \ldots, \xx_m\}\subseteq [-1, 1]^n$ and a generator $g$ constructed by a generator-constructing algorithm. If $|g(\xx)|$ is small for  $\xx\in X$, we expect $|g(\yy)|$ to also be small for $\yy\in[-1, 1]^n$ that is close to $\xx$. As we demonstrate below, this is indeed the case for \oavi{} solving \eqref{eq:ccop} for the norm $\|\cdot\|_1$ and $\tau \geq 2$.

\begin{lemma}[A simple learning guarantee]\label{lemma:simple_learning_guarantee}
Let $X = \{\xx_1, \ldots, \xx_m\} \subseteq [-1, 1]^n$ and let $\tau \geq 2$.
Let $1 > \psi \geq \eps \geq 0$ and let $(\cG, \cO) = \oavi{}(X, \psi, \eps)$ be the output of running \oavi{} solving \eqref{eq:ccop} for the norm $\|\cdot\|_1$ and $\tau$.
Then, for any $\yy \in [-1, 1]^n$, $\xx\in X$, and $g\in \cG$ of degree $d\in \N$, it holds that
$
    |g(\yy)| \leq |g(\xx)|  + d\tau \|\yy - \xx \|_\infty.
$
\end{lemma}
\begin{proof}
Let $g = \sum_{i = 1}^k c_i t_i \in \cG$ be a polynomial of degree $d\in \N$, where $k\in \N$ and $c_i \in \R$ and $t_i \in \cT$ for all $i \in \{1, \ldots, k\}$. By the mean value theorem, for any $\xx, \yy\in [-1, 1]^n$, 
$
    |g(\yy)| \leq |g(\xx)| + \max_{\zz \in[-1, 1]^n} |\langle \nabla f(\zz), \yy - \xx \rangle|.
$
By the definition of the dual norm,
$
    |g(\yy)| \leq |g(\xx)| + \max_{\zz \in[-1, 1]^n}\|\nabla g(\zz)\|_1 \|\yy - \xx \|_\infty.
$
Since $g$ is of degree $d$ and $\|g\|_1 =\|\cc\|_1 \leq \tau$, it holds that $\max_{\zz \in[-1, 1]^n}\|\nabla g(\zz)\|_1 \leq d \tau$ and the result follows.
\end{proof}

Since \abm{}, \avi{}, \bbabm{}, \vca{}, and \oavi{} solving \eqref{eq:cop} do not construct generators with coefficient vectors that are bounded in $\ell_1$-norm, Lemma~\ref{lemma:simple_learning_guarantee} does not apply to these algorithms.
For the remainder of this section, we derive two additional learning guarantees for \oavi{} solving \eqref{eq:ccop} for the norm $\|\cdot\|_1$ and $\tau \geq 2$ that also rely on the coefficient vectors of generators to be bounded in $\ell_1$-norm. Thus, the upcoming generalization bounds do not apply to \abm{}, \avi{}, \bbabm{}, \vca{}, and \oavi{} solving \eqref{eq:cop}, making our algorithm a theoretically better supported alternative to related generator-constructing techniques. Deriving learning guarantees for other generator-constructing algorithms remains an open problem.

\subsection{Generators Vanish over Out-Sample Data}\label{section:vanishing_property}
In this section, under mild assumptions, we prove that generators constructed by \oavi{} not only vanish over in-sample (or training) data but also over out-sample data. 
Let $\cX \subseteq [-1, 1]^n$, let $k \in \N_{> 0}$, and let $\tau > 0$. The following hypothesis class captures generators constructed with \oavi{} solving \eqref{eq:ccop} for the norm $\|\cdot\|_1$ and $\tau$ terminated early to guarantee that $|\cO| \leq k - 1$\footnote{The assumption $|\cO| \leq k - 1\in \N$ can, for example, be achieved by terminating \oavi{} when a certain degree is reached.
}:
\begin{align}\tag{HC-generators}\label{eq:generators_hc}
    \cH & = \left\{ \cX \ni \xx \mapsto \cc^\intercal \cU(\xx) \colon \|\cc\|_1 \leq \tau, \cU\in \Xi_{k} \right\}.
\end{align}
Below, we compute the empirical \emph{Rademacher complexity} of the hypothesis class $\cH$ as in \eqref{eq:generators_hc}.

\begin{lemma}
[Rademacher complexity of \ref{eq:generators_hc}]
\label{lemma:rademacher_complexity_generators}
Let $\cX \subseteq [-1, 1]^n$, let $\tau > 0$, let $k\in \N_{> 0}$, let $\cH$ be as in \eqref{eq:generators_hc}, and let 
$X = (\xx_1, \ldots, \xx_m) \in \cX^m$ be drawn i.i.d. according to a distribution $\cD$. Then, the empirical Rademacher complexity of $\cH$ is bounded as follows:
$
    \hat{\cR}_X(\cH) \leq \tau  \sqrt{2\log (2  k |\Xi_{k}|) / m}.
$
\end{lemma}
\begin{proof}
The proof follows the line of arguments of \citet[Theorem 11.15]{mohri2018foundations}, modified to our setting.
Let $\ssig = (\sig_1, \ldots, \sig_k)^\intercal\in \{-1, 1\}^k$ be a vector of uniform random variables. It holds that
\begin{align*}
    \hat{\cR}_X(\cH) & = \frac{1}{m}\E_\ssig \left[\sup_{\|\cc\|_1 \leq\tau} \sup_{\cU \in \Xi_{k}} \sum_{i = 1}^m \sig_i \cc^\intercal \cU(\xx_i) \right] \\
    & = \frac{\tau}{m}\E_\ssig \left[\sup_{\cU \in \Xi_{k}} \left\|\sum_{i = 1}^m \sig_i \cU(\xx_i) \right\|_\infty\right] & \text{$\triangleright$ by the definition of the dual norm}\\
    & = \frac{\tau}{m}\E_\ssig \left[ \sup_{\cU \in \Xi_{k}} \max_{j\in \{1, \ldots, k\}}\max_{s\in \{-1,1\}} s \sum_{i = 1}^m \sig_i  \cU(\xx_i)_j \right] & \text{$\triangleright$ by the definition of $\|\cdot\|_\infty$ and $\abs{\cdot}$}\\
    & = \frac{\tau}{m}\E_\ssig \left[ \sup_{\zz \in A} \sum_{i = 1}^m \sig_i  z_i \right],
\end{align*}
where $A = \left\{s (\cU(\xx_1)_j,\ldots, \cU(\xx_m)_j)^\intercal \colon \cU \in \Xi_{k}, j\in \{1, \ldots, k\}, s\in \{-1, 1\}\right\}$. Since $\|\cU(\xx)\|_\infty \leq 1$ for all $\xx \in \cX \subseteq [-1, 1]^n$, for any $\zz \in A$, it holds that $\|\zz\|_2 \leq \sqrt{m}$. By \citet[Theorem 3.7]{mohri2018foundations}, since $A$ contains at most $2 k |\Xi_k|$ elements, we have
$
    \hat{\cR}_X(\cH) \leq \tau  \sqrt{2\log (2  k |\Xi_{k}|) / m}.
$
\end{proof}
Under mild assumptions, generators constructed by \oavi{} are contained in the hypothesis class \eqref{eq:generators_hc} and, in expectation, vanish approximately over both in-sample and out-sample data. 
\begin{theorem}
[Vanishing property]
\label{thm:vanishing_property}
Let $\cX \subseteq [-1, 1]^n$, let $\tau \geq 2$, let $k\in \N_{>0}$, and let
$X = (\xx_1, \ldots, \xx_m) \in \cX^m$ be drawn i.i.d. according to a distribution $\cD$.
Let $1 > \psi \geq \eps \geq 0$ and let $(\cG, \cO) = \oavi{}(X, \psi, \eps)$ be the output of running \oavi{} solving \eqref{eq:ccop} for the norm $\|\cdot\|_1$ and $\tau$ terminated early to guarantee that $|\cO| \leq k - 1$. 
Then, for any $\delta > 0$, with probability at least $1 - \delta$, the following inequality holds for all $g\in \cG$:
\begin{align*}
    \E_{\xx \backsim \cD}\left[\mse(g,\{\xx\})\right]\leq \mse(g, X) + 4 \tau^2  \sqrt{\frac{2\log (2  k |\Xi_{k}|)}{m}} + 12 \tau^2 \sqrt{\frac{\log (2 \delta^{-1})}{2m}}.
\end{align*}
\end{theorem}
\begin{proof}

For all $h\in \cH$ and $\xx\in \cX$, it holds that $|h(\xx)| \leq \tau$. Thus, plugging Lemma~\ref{lemma:rademacher_complexity_generators} into Theorem~11.3 of \citet{mohri2018foundations} with the $2\tau$-Lipschitz continuous loss function $L(y, y'):= |y - y'|^2$ for $y, y' \in [-\tau, \tau]$ implies that for any $\delta > 0$, with probability at least $1 - \delta$, the following inequality holds for all $h\in \cH$:
$
     \E_{\xx \backsim \cD}\left[h(\xx)^2\right]\leq \frac{1}{m} \sum_{i = 1}^m h(\xx_i)^2 + 4 \tau^2  \sqrt{2\log (2  k |\Xi_{k}|) / m} + 12 \tau^2 \sqrt{\log (2 \delta^{-1}) / 2m}.
$
Let $j, k \in \N$ with $j<k$. Since for any $\cV\in \Xi_j$, there exists $\cU \in \Xi_k$ such that $\cV \subseteq \cU$, any $g\in \cG$ can be written in the form
$g = \sum_{i = 1}^{k} c_i t_i$, where $\|\cc\|_1 \leq \tau$ and $t_1, \ldots, t_k \in \cU \in \Xi_{k}$. Thus, $\cX\ni \xx \mapsto g(\xx) = \cc^\intercal \cU(\xx)$ is contained in $\cH$, proving the theorem.
\end{proof}

When \oavi{} solving \eqref{eq:ccop} for the norm $\|\cdot\|_1$ and $\tau \geq 2$  is terminated early to guarantee $|\cO|\leq k-1$ for $k\in\N_{>0}$ independent of $m$, plugging Lemma~\ref{lemma:coarse_Xi} into Theorem~\ref{thm:vanishing_property} results in an explicit generalization bound on the extent of vanishing of constructed generators.

\subsection{Margin Bound for \oavi{} with Linear Kernel \svm{}}\label{section:margin_bound}
In this section, we derive a margin bound for using $\tau$-bounded generators of the approximate vanishing ideal to transform features for a linear kernel \svm{}.
We require the following two definitions from \citet{mohri2018foundations}.
\begin{definition}
[Margin loss function]
For any $\rho > 0$, the \emph{$\rho$-margin loss} is the function $L_\rho \colon \R \times \R \to \R_{>0}$ defined for all $y, y'\in \R$ as
$
    L_\rho(y, y') = \min \{ 1, \max\left\{0, 1 - y y'/\rho \right\} \}.
$
\end{definition}
\begin{definition}
[Empirical margin loss]
Let $\cX \subseteq [-1, 1]^n$, let $\cY = \{-1, +1\}$, let 
$X=(\xx_1,\ldots,\xx_m)\in\cX^m$ be drawn i.i.d. according to a distribution $\cD$, let $f\colon \cX\to \cY$ be a target function, and let $h\colon\cX\to\cY$ be a hypothesis. The \emph{empirical margin loss} of $h$ is defined as 
$
    \hat{R}_{X, f, \rho}(h) = \frac{1}{m} \sum_{i = 1}^m L_\rho (h(\xx_i), f(\xx_i)).
$
\end{definition}

Recall the machine learning pipeline for classification explained in Section~\ref{section:the_cgavi_pipeline}. For ease of exposition, we restrict ourselves to the binary classification setting. Let $\cX \subseteq [-1, 1]^n$, let $\cY = \{-1, +1\}$, let $k\in \N_{>0}$, and let $\tau \geq 2$. We are given a training sample $X = (\xx_1, \ldots, \xx_m)\in\cX^m$ drawn $i.i.d.$ from some unknown distribution $\cD$ and a target function $f\colon\cX\to\cY$. For samples $X^{\pm 1}\subseteq X$ belonging to classes $\pm 1$, we construct sets of generators $\cG^{\pm 1}$ using \oavi{} solving \eqref{eq:ccop} for the norm $\|\cdot\|_1$ and $\tau$ terminated early to guarantee that $|\cO| \leq k - 1$. In other words, for $1>\psi\geq \eps \geq 0$, let $(\cG^{\pm1}, \cO^{\pm1}) = \oavi{}(X^{\pm 1}, \psi, \eps)$ such that $|\cO^{\pm1}| \leq k-1$ and $\|g\|_1 \leq \tau$ for all $g\in \cG := \cG^{-1} \cup \cG^{+1}$. Then, $|\cG^{\pm 1}|\leq kn$, $|\cG|\leq 2kn$, and any $g\in \cG$ can be written as $g=\sum_{i=1}^k c_i t_i$, where $\|\cc\|_1\leq \tau$ and $t_1, \ldots, t_k \in \cU$ for some $\cU \in \Xi_k$.
Then, we transform samples $\xx \in X$ with the feature transformation \eqref{eq:feature_transformation} associated with $\cG$ and apply a linear kernel \svm{} to the feature-transformed data. This approach is captured by the hypothesis class
\begin{align}\tag{HC}\label{eq:rademacher_hc}
    \cH = \left\{\cX \ni \xx \mapsto \ww^\intercal \begin{pmatrix}
    \left\lvert\cc_1^\intercal \cU_1(\xx)\right\rvert \\
    \vdots \\
    \left\lvert\cc_{2kn}^\intercal \cU_{2kn}(\xx)\right\rvert
    \end{pmatrix} 
    \colon \|\ww\|_{1} \leq \Lambda, \|\cc_j\|_1\leq \tau \ \text{and} \ \cU_j \in \Xi_{k} \ \text{for all} \ j\in [2kn] \right\},
\end{align}
where $0 < \Lambda \in \R$ and $[2kn]:= \{1, \ldots, 2kn\}$.
Below, we bound the empirical Rademacher complexity of \eqref{eq:rademacher_hc}. 

\begin{lemma}
[Rademacher complexity of \ref{eq:rademacher_hc}]
\label{lemma:rademacher_complexity}
Let $\cX \subseteq [-1, 1]^n$, let $\Lambda >0$, let $\tau \geq 2$, let $k\in \N_{>0}$, let $\cH$ be as in \eqref{eq:rademacher_hc}, and let 
$X = (\xx_1,\ldots, \xx_m)\in\cX^m$ be drawn i.i.d. according to a distribution $\cD$.
Then, the empirical Rademacher complexity of $\cH$ is bounded as follows:
$
    \hat{\cR}_X(\cH) \leq 2\Lambda\tau  \sqrt{2\log (2  k |\Xi_{k}|)/m}.
$
\end{lemma}
\begin{proof}
The proof is an adaptation of the proof of \citet[Theorem~6.12]{mohri2018foundations} to the hypothesis class $\cH$ as in \eqref{eq:rademacher_hc}. Let $\ssig = (\sig_1, \ldots, \sig_m)^\intercal\in \{-1, 1\}^m$ be a vector of uniform random variables. 
We write $\sup_{\substack{\|\cc_j\|_1\leq \tau \\ \forall j\in [2kn] }}$ for $\sup_{\|\cc_1\|_1\leq \tau} \cdots \sup_{\|\cc_{2kn}\|_1\leq \tau}$ and $\sup_{\substack{\cU_j\in \Xi_k\\ \forall j\in [2kn] }}$ for $\sup_{\cU_1\in \Xi_k} \cdots \sup_{\cU_{2kn}\in \Xi_k}$.
It holds that
\begin{align*}
    \hat{\cR}_X(\cH) & = \frac{1}{m}\E_\ssig \left[\sup_{\|\ww\|_1 \leq \Lambda} \sup_{\substack{\|\cc_j\|_1\leq \tau \\ \forall j\in [2kn] }} \sup_{\substack{\cU_j\in \Xi_k \\ \forall j\in [2kn] }} \sum_{i = 1}^m \sig_i\ww^\intercal  \begin{pmatrix}
    \left\lvert\cc_1^\intercal \cU_1(\xx_i)\right\rvert \\
    \vdots \\
    \left\lvert\cc_{2kn}^\intercal \cU_{2kn}(\xx_i)\right\rvert
    \end{pmatrix} \right]\\
    & \leq \frac{\Lambda}{m}\E_\ssig \left[\sup_{\substack{\|\cc_j\|_1\leq \tau \\ \forall j\in [2kn] }} \sup_{\substack{\cU_j\in \Xi_k \\ \forall j\in [2kn] }} \left\|\sum_{i = 1}^m \sig_i \begin{pmatrix}
    \left\lvert\cc_1^\intercal \cU_1(\xx_i)\right\rvert \\
    \vdots \\
    \left\lvert\cc_{2kn}^\intercal \cU_{2kn}(\xx_i)\right\rvert
    \end{pmatrix}\right\|_\infty \right] & \text{$\triangleright$ by the definition of the dual norm}\\
    & = \frac{\Lambda}{m}\E_\ssig \left[\sup_{\|\cc\|_1\leq \tau} \sup_{\cU\in \Xi_k } \left\lvert\sum_{i = 1}^m \sig_i \left\lvert\cc^\intercal \cU(\xx_i)\right\rvert\right\rvert \right] & \text{$\triangleright$ by the definition of $\|\cdot\|_\infty$}\\
    & \leq \frac{2\Lambda}{m}\E_\ssig \left[\sup_{\|\cc\|_1\leq \tau} \sup_{\cU\in \Xi_k } \left\lvert\sum_{i = 1}^m \sig_i \cc^\intercal \cU(\xx_i)\right\rvert \right] & \text{$\triangleright$ by \citet{ledoux1991probability}}\\
    & \leq \frac{2\Lambda\tau}{m}\E_\ssig \left[\sup_{\cU\in \Xi_{k} } \left\|\sum_{i = 1}^m \sig_i \cU(\xx_i)\right\|_\infty \right] & \text{$\triangleright$ by the definition of the dual norm}\\
    & = \frac{2\Lambda\tau}{m}\E_\ssig \left[\sup_{\cU\in \Xi_{k} } \max_{j\in \{1, \ldots, k\}} \max_{s\in \{-1,1\}} s \sum_{i = 1}^m \sig_i (\cU(\xx_i))_j \right] & \text{$\triangleright$ by the definition of $\|\cdot\|_\infty$ and $\lvert\cdot\rvert$}\\
    & = \frac{2\Lambda\tau}{m}\E_\ssig \left[\sup_{\zz\in A} \sum_{i = 1}^m \sig_i z_i\right], 
    \end{align*}
where $A = \left\{s (\cU(\xx_1)_j,\ldots, \cU(\xx_m)_j)^\intercal \colon \cU \in \Xi_{k}, j\in \{1, \ldots, k\}, s\in \{-1, 1\}\right\}$. Since $\|\cU(\xx)\|_\infty \leq 1$ for all $\xx \in \cX \subseteq [-1, 1]^n$, for any $\zz \in A$ it holds that $\|\zz\|_2 \leq \sqrt{m}$. By \citet[Theorem 3.7]{mohri2018foundations}, since $A$ contains at most $2 k |\Xi_{k}|$ elements, we have
$
    \hat{\cR}_X(\cH) \leq 2\Lambda\tau  \sqrt{2\log (2  k |\Xi_{k}|) / m}.
$
\end{proof}

Lemma~\ref{lemma:rademacher_complexity}, in combination with \citet[Theorem~5.8]{mohri2018foundations}, implies the following margin bound and the observation that the \emph{true risk} is $R(h) = \E_{\xx\backsim \cD} [\chi_{f(\xx)h(\xx) \leq 0}] = \P_{\xx \backsim \cD}[\sgn(h (\xx)) \neq f(\xx)]$, where $\chi_\omega$ is the indicator function of the event $\omega$ and $\sgn(\cdot)$ is the sign function.

\begin{theorem}[Margin bound]\label{thm:margin_bound}
Let $\cX \subseteq [-1, 1]^n$, let $\cY = \{-1, +1\}$, let $\Lambda >0$, let $\tau \geq 2$, let $k\in \N_{>0}$, let $\cH$ be as in \eqref{eq:rademacher_hc}, and
let $X = (\xx_1, \ldots, \xx_m)\in \cX^m$ be drawn i.i.d. according to a distribution $\cD$, and let $f\colon \cX \to\cY$ be a target function.
Fix $\rho > 0$. Then, for any $\delta > 0$, with probability at least $1 - \delta$, the following inequality holds for all $h\in \cH$:
\begin{align*}
    \P_{\xx \backsim \cD}[\sgn(h (\xx)) \neq f(\xx)] \leq \hat{R}_{X,f,\rho} (h) + \frac{4\Lambda \tau}{\rho} \sqrt{\frac{2\log (2  k |\Xi_{k}|)}{m}} + 3 \sqrt{\frac{\log (2 \delta^{-1})}{2m}}.
\end{align*}
\end{theorem}
Plugging Lemma~\ref{lemma:coarse_Xi} into Theorem~\ref{thm:margin_bound} yields an explicit margin bound for features transformed with \oavi{} solving \eqref{eq:ccop} for the norm $\|\cdot\|_1$ and $\tau \geq 2$ terminated early to guarantee that $|\cO| \leq k - 1$ for $k\in\N_{>0}$ independent of $m$ and a subsequently applied linear kernel \svm{}.

\section{Conditional Gradients}\label{section:conditional_gradients}

\begin{algorithm}[t]
\SetKwInOut{Input}{Input}\SetKwInOut{Output}{Output}
\SetKwComment{Comment}{$\triangleright$\ }{}
\Input{A smooth and convex function $f$, a set of atoms $\cA\subseteq \R^n$, a vertex $\xx_0 \in \cA$, and $T\in \N$.}
\Output{A point $\xx_T\in \conv(\cA)$.}
\hrulealg
{$S^{(0)}\gets \{ \xx_0\}$}\\
{$\lambda_{\vv}^{(0)} \gets 1$ for $\vv = \xx_0$ and $0$ otherwise}\\
\For{$t = 0,\ldots, T-1$}{
 {$ \ss_t\gets \argmin_{\ss \in \cA} \langle \nabla f(\xx_t), \ss \rangle$}\label{alg:pfw_fw_vertex}\Comment*[f]{FW vertex}\\
 {$ \dd^{FW}_{t} \gets \ss_t - \xx_t$}\\
 {$ \vv_t\gets \argmax_{\vv\in S^{(t)}}\langle \nabla f(\xx_t), \vv\rangle$}\Comment*[f]{away vertex}\\
 {$ \dd^{AW}_{t} \gets  \xx_t - \vv_t$}\label{alg:pfw_aw_vertex}\\
 {$\dd_t \gets \dd^{FW}_{t} + \dd^{AW}_{t}$}\label{alg:pfw_pw_direction}\\
 {$\gamma_t = \argmin_{\gamma \in [0,\lambda_{\vv_t}]} f(\xx_t + \gamma \dd_t)$}\Comment*[f]{line search}\\
 {$\xx_{t+1} \gets \xx_t + \gamma_t \dd_t$}\\
 {$\lambda_{\vv}^{(t+1)} \gets \lambda^{(t)}_{\vv}$ for $\vv\in\cA\setminus\{\ss_t, \vv_t\}$}\\
 {$\lambda_{\ss_t}^{(t+1)} \gets \lambda^{(t)}_{\ss_t} + \gamma_t$}\\
 {$\lambda_{\vv_t}^{(t+1)} \gets \lambda^{(t)}_{\vv_t} - \gamma_t$}\\
 {$S^{(t+1)} \gets \{\vv\in \cA \mid \lambda_{\vv}^{(t+1)}>0\}$} \Comment*[f]{update active set}
}
\caption{Pairwise Conditional Gradients Algorithm (\pcg{})}
 \label{algorithm:pairwise_frank_wolfe}
\end{algorithm}

Any $\eps$-accurate solution to \eqref{eq:cop} or \eqref{eq:ccop} can be used to construct the polynomial $g$ returned by \cvxoracle{}, allowing the practitioner to choose the best solver for any given task.
Our goal is to construct a set of generators $\cG$ consisting of few and sparse polynomials to obtain a compact representation of the approximate vanishing ideal.
The former property is achieved by restricting leading terms of generators to be in the border. We address the latter property with our choice of solver for \eqref{eq:cop} or \eqref{eq:ccop}. 

\subsection{Sparsity}\label{section:sparsity}
To do so, we formalize the notion of sparsity. 
Consider the execution of \oavi{} and suppose that, currently, $\cO = \{t_1, \ldots, t_k\}_\sig$ and $g = \sum_{i = 1}^{k} c_i t_i + t$, where $k\in \N$ and $c_i \in \R$ and $t_i \in \cT$ for all $i \in \{1, \ldots, k\}$, with $\lt(g) = t\not\in \cO$ gets appended to $\cG$. Let the number of entries, the number of zero entries, and the number of non-zero entries in the coefficient vector of $g$ be denoted by $g_e:= k$, $g_z:= |\{c_i = 0 \colon i\in \{1, \ldots, k\}\}|$, and $g_n:= g_e - g_z$, respectively.
We define the \emph{sparsity} of $g$ as
$
\spar(g):= g_z/g_e \in [0, 1].
$
Larger $\spar(g)$ indicates a more thinly populated coefficient vector of $g$.
Recall that for classification as in Section~\ref{section:the_cgavi_pipeline} we construct sets of generators $\cG^i$ corresponding to classes $i$ and then transform $X$ via \eqref{eq:feature_transformation} using all polynomials in $\cG := \bigcup_i \cG^i$. We define the \emph{sparsity} of $\cG$ as
\begin{equation}\tag{SPAR}\label{eq:sparsity}
    \spar (\cG) : = \Big(\sum_{g \in \cG}g_z\Big) / \Big(\sum_{g \in \cG} g_e\Big)\in[0, 1].
\end{equation}
To quickly create sparse generators, we implement \cvxoracle{} with the pairwise conditional gradients algorithm (\pcg{}), see Algorithm~\ref{algorithm:pairwise_frank_wolfe}.

\subsection{The Pairwise Conditional Gradients Algorithm (\pcg{})}\label{section:pfw}
Let $f\colon \R^n \to \R$ be a convex and smooth function and $\cA = \{\vv_1, \ldots, \vv_k\}\subseteq \R^n$ a set of vectors. Suppose that $f$ is differentiable in an open set containing $\conv(\cA)$, where $\conv(\cA)$ is the convex hull of $\cA$.  Conditional gradients algorithms (\cg) are a family of methods that solve
\begin{align}\label{eq:cg_opt}
    \xx^*\in \min_{\xx \in \conv(\cA)} f(\xx).
\end{align}
For an $L$-smooth objective and feasible region $\conv(\cA)$ with diameter $\delta > 0$, vanilla \cg{} with line-search step-size rule constructs an $\eps$-accurate solution to \eqref{eq:cg_opt} in less than $\frac{2L\delta^2}{\eps}$ iterations \citep{jaggi2013revisiting}. There are several algorithmic variants of \cg{} that further improve sparsity. Here, we focus on the pairwise conditional gradients algorithm (\pcg{}), an algorithmic variant of \cg{} known for its tendency to produce highly sparse solutions when solving \eqref{eq:ccop}. Another benefit of using \pcg{} is that the algorithm converges linearly for \eqref{eq:ccop}, albeit with constants that strongly depend on the problem dimension \citep{lacoste2015global}. In our numerical experiments, the dependence on the dimension of the problem does not cause issues.

For better understanding of \pcg{} and how the method constructs sparse iterates, we present a short overview of Algorithm~\ref{algorithm:pairwise_frank_wolfe}.
At iteration $t = 0,\ldots, T-1$, \pcg{} writes the current iterate, $\xx_t$, as a convex combination of elements of $\cA$, that is, $\xx_t = \sum_{\vv \in \cA}\lambda_\vv^{(t)} \vv$, where $\sum_{\vv \in \cA}\lambda_\vv^{(t)} = 1$ and $\lambda_\vv^{(t)}\in [0, 1]$ for all $\vv\in \cA$. At iteration $t$, a vertex $\vv$ whose corresponding weight $\lambda_\vv^{(t)}$ is not zero is referred to as an \emph{active vertex} and the set $S^{(t)} =  \{\vv\in \cA \mid \lambda_{\vv}^{(t)}>0\}$ is referred to as the \emph{active set} at iteration $t$. During each iteration, \pcg{} determines two vertices requiring access to a first-order oracle and a linear minimization oracle. 
In Line~\ref{alg:pfw_fw_vertex}, \pcg{} determines the Frank-Wolfe vertex, which minimizes the scalar product with the gradient of $f$ at iterate $x_t$. Taking a step of appropriate size towards the Frank-Wolfe vertex, in the Frank-Wolfe direction, reduces the objective function value. In Line~\ref{alg:pfw_aw_vertex}, \pcg{} determines the away vertex in the active set, which maximizes the scalar product with the gradient of $f$ at iterate $x_t$. Taking a step away from the away vertex, in the away direction, reduces the objective function value. In Line~\ref{alg:pfw_pw_direction}, \pcg{} combines the away direction and the Frank-Wolfe direction into the pairwise direction and takes a step with optimal step size in the pairwise direction, shifting weight from the away vertex to the Frank-Wolfe vertex. In each iteration, \pcg{} thus only modifies two entries of the iterate $\xx_t$, which is the main reason why \pcg{} tends to return a sparse iterate $\xx_T$.
When using \pcg{} as \cvxoracle{} for \oavi{}, we implement \cvxoracle{} as follows: Run \pcg{} with $f$, $\cA$ the set of vertices of the $\ell_1$-ball of radius $\tau  - 1$, $\xx_0 = (\tau - 1,0 \ldots, 0)^\intercal$,  and $T \in \N$ such that \pcg{} achieves $\eps$-accuracy to obtain $\xx_T$. Then, $(\xx_T^\intercal, 1)^\intercal$ is the coefficient vector of the polynomial $g$ returned by \cvxoracle{}.

\section{On Borders}\label{sec:note_on_borders}

In this section, we focus on the border defined in Definition~\ref{definition:border} and how it compares to borders used in other generator-constructing algorithms.

We first compare the border used in \oavi{} to the border used in \vca{}. Since \vca{} is monomial-agnostic, the set corresponding to $\cO$ in \vca{} is not a set of monomials but a set of polynomials that provably do not vanish approximately over the data, and the border is defined as
\begin{align*}
    \partial_d\cO := \{u = v \cdot t \in \cT_d \mid v\in \cO_{1}, t \in \cO_{d-1} \}.
\end{align*}

Thus \vca{}'s border is a superset of \oavi{}'s border and can lead to the construction of unnecessary generators. Suppose, for example, that during the execution of the algorithms, it holds that $\cO_1 = \{t_1, \ldots, t_n\}\subseteq \cT$ and $\cO_2 = \{t_1^2\}$. 
For \oavi{}, $\partial_3 \cO= \{t_1^3\}$, and for \vca{}, $\partial_3\cO = \{t_1^3, t_1^2t_2\ldots, t_1^2t_n\}$. In this example, \vca{} can construct up to $n-1$ redundant generators. 
Since \vca{} is monomial-agnostic, \vca{}'s border cannot be replaced by the border as defined in Definition~\ref{definition:border}.

To describe the differences between the borders of \oavi{} and other monomial-aware algorithms, we recall the original algebraic structures that motivate the different borders. Let $X \subseteq \R^n$, $\psi = \eps = 0$, and $(\cG, \cO) = \oavi(X, \psi, \eps)$. Then, $\cG$ forms a particular type of generating set for an ideal, a so-called \emph{reduced Gröbner basis}, see \citet{kreuzer2000computational} for the technical definition. We thus refer to the border in Definition~\ref{definition:border} as the \emph{reduced Gröbner basis border} (border-\ref{eq:gb}). As we showed in Lemma~\ref{lemma:order_ideal}, employing the border-\eqref{eq:gb} in \oavi{} guarantees that the output of \oavi{} forms an order ideal. 
Other monomial-aware algorithms such as \abm{}, \avi{}, and \bbabm{} forgo the border-\eqref{eq:gb} in favour of the \emph{border basis border} (border-\ref{eq:bb}):
\begin{align}\label{eq:bb}\tag{\boba{}}
    \partial_d\cO := \{u = v \cdot t \in \cT_d \mid v\in \cT_{1}, t \in \cO_{d-1} \}.
\end{align}
Note that \oavi{}, \abm{}, \avi{}, and \bbabm{} can all be run with the border-\eqref{eq:gb} or the border-\eqref{eq:bb} and \oavi{} and \bbabm{} are equivalent when run with the same border.
The generating sets constructed with \abm{}, \avi{}, and \bbabm{} form \emph{border bases}, which tend to be more robust to perturbations in the data than reduced Gröbner bases \citep{limbeck2013computation}. However, foregoing the border-\eqref{eq:gb} also leads to the construction of more generators.
In Figure~\ref{fig:border}, we compare the number of constructed generators for \pcg{}\avi{} and \abm{} for varying vanishing parameters $\psi > 0$ for different data sets. The plots indicate that when the algorithms are run with the border-\eqref{eq:bb}, they tend to construct more generators than with the border-\eqref{eq:gb}.

\begin{figure}[t]
\centering
\begin{tabular}{c c c c}
\begin{subfigure}{.22\textwidth}
    \centering
        \includegraphics[width=1\textwidth]{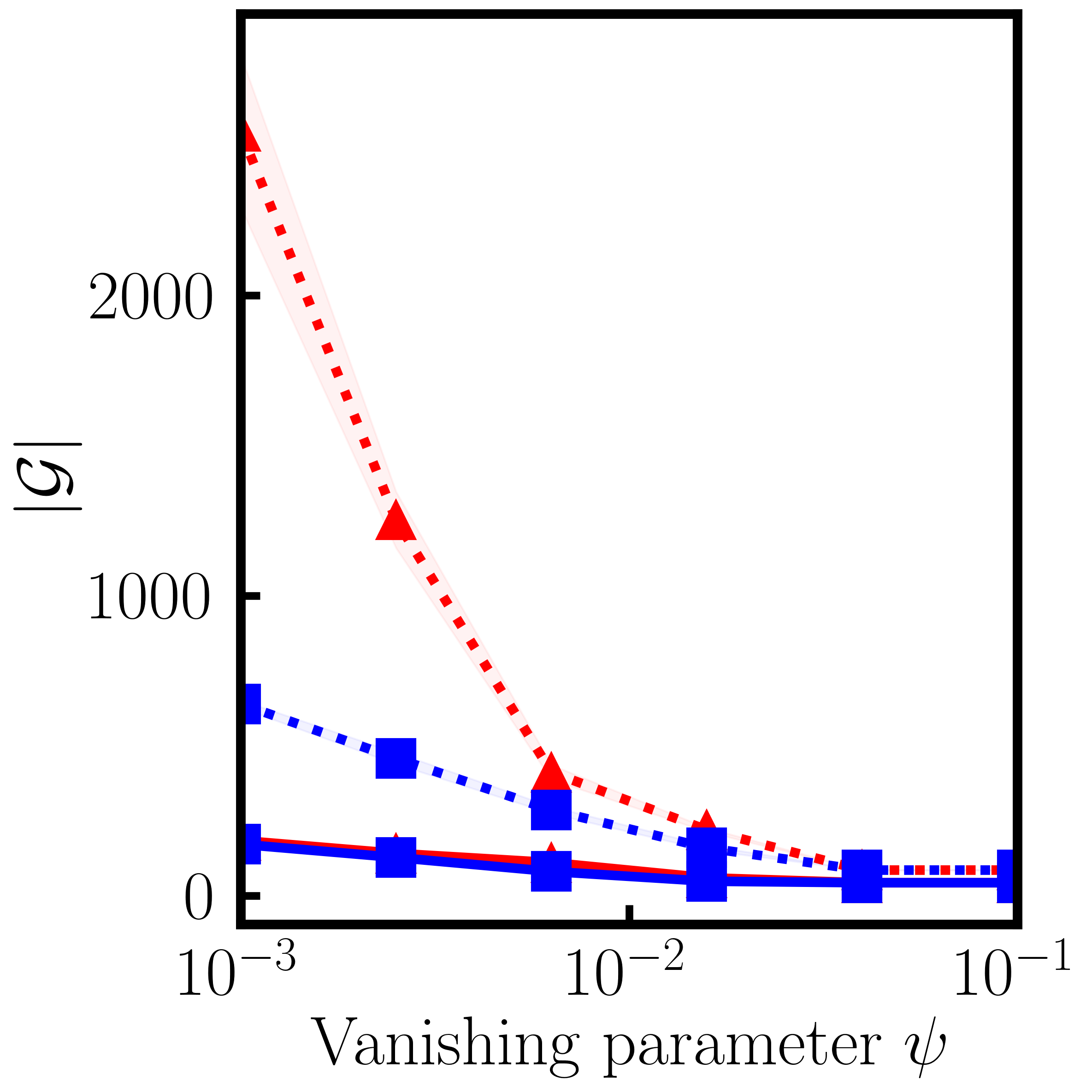}
        \caption{credit}
        \label{fig:border_credit}
    \end{subfigure}
    & 
    \begin{subfigure}{.22\textwidth}
    \centering
        \includegraphics[width=1\textwidth]{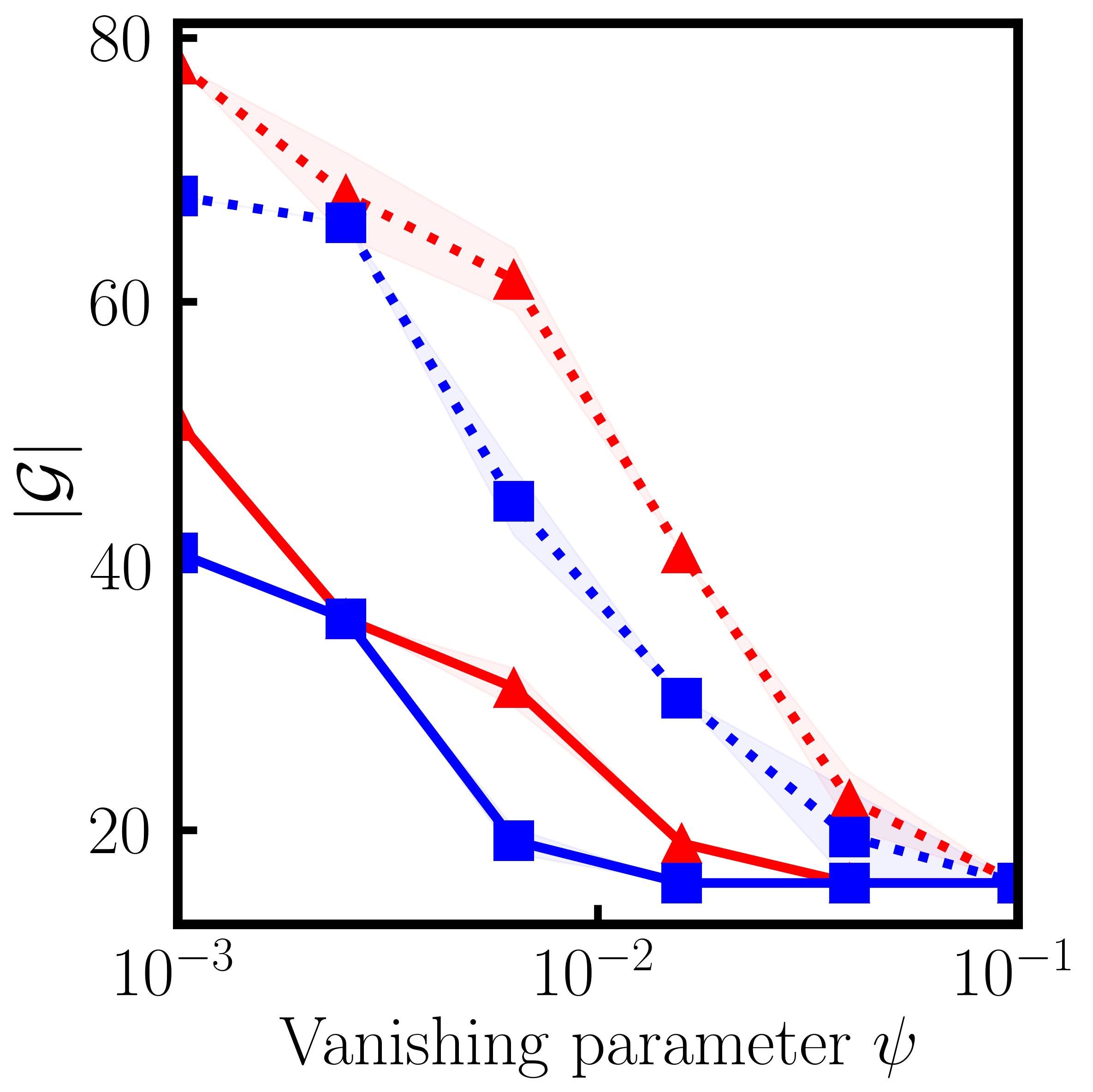}
        \caption{htru}
        \label{fig:border_htru}
    \end{subfigure}
    & 
    \begin{subfigure}{.22\textwidth}
    \centering
        \includegraphics[width=1\textwidth]{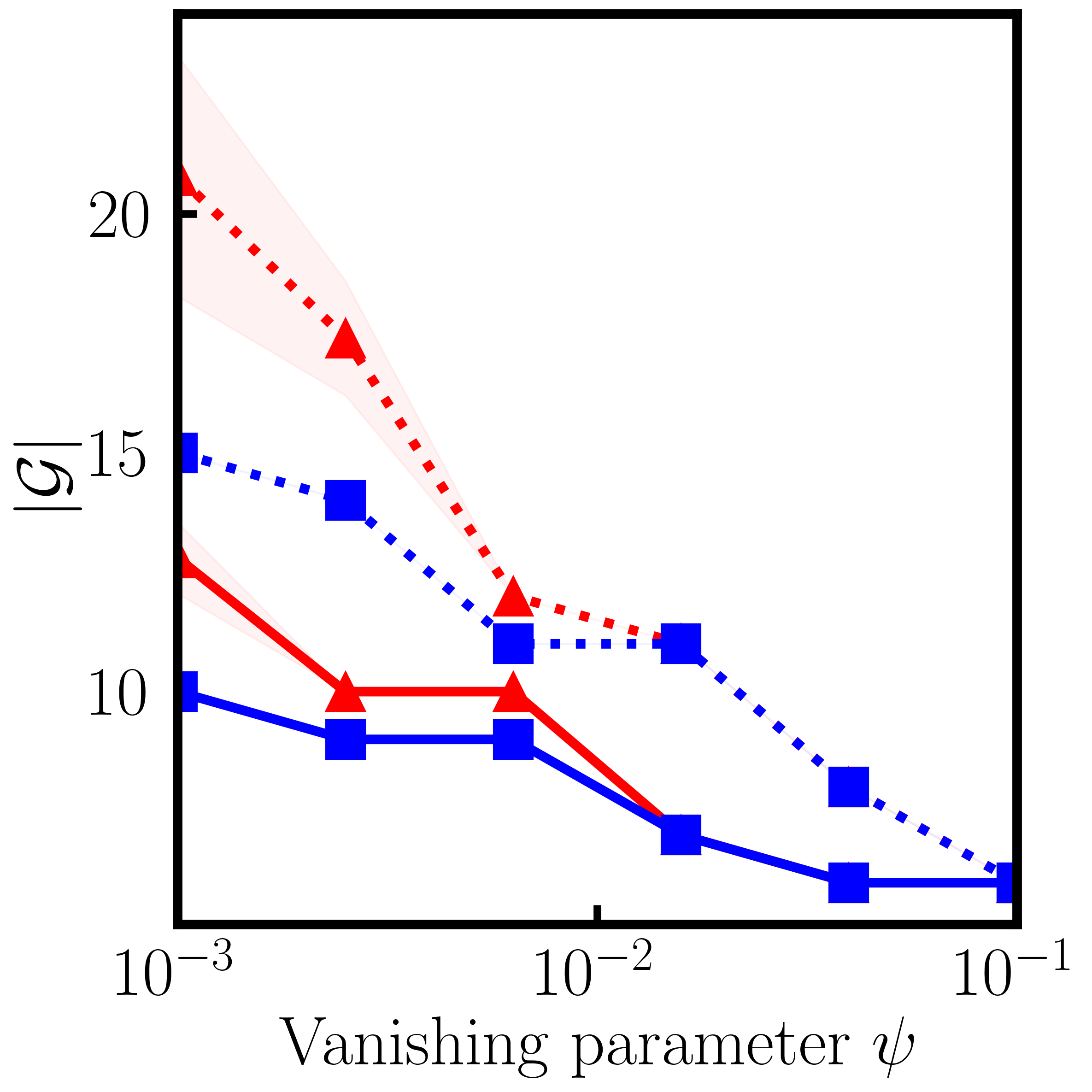}
        \caption{skin}
        \label{fig:border_skin}
    \end{subfigure} 
    & 
    \begin{subfigure}{.22\textwidth}
    \centering
        \includegraphics[width=1\textwidth]{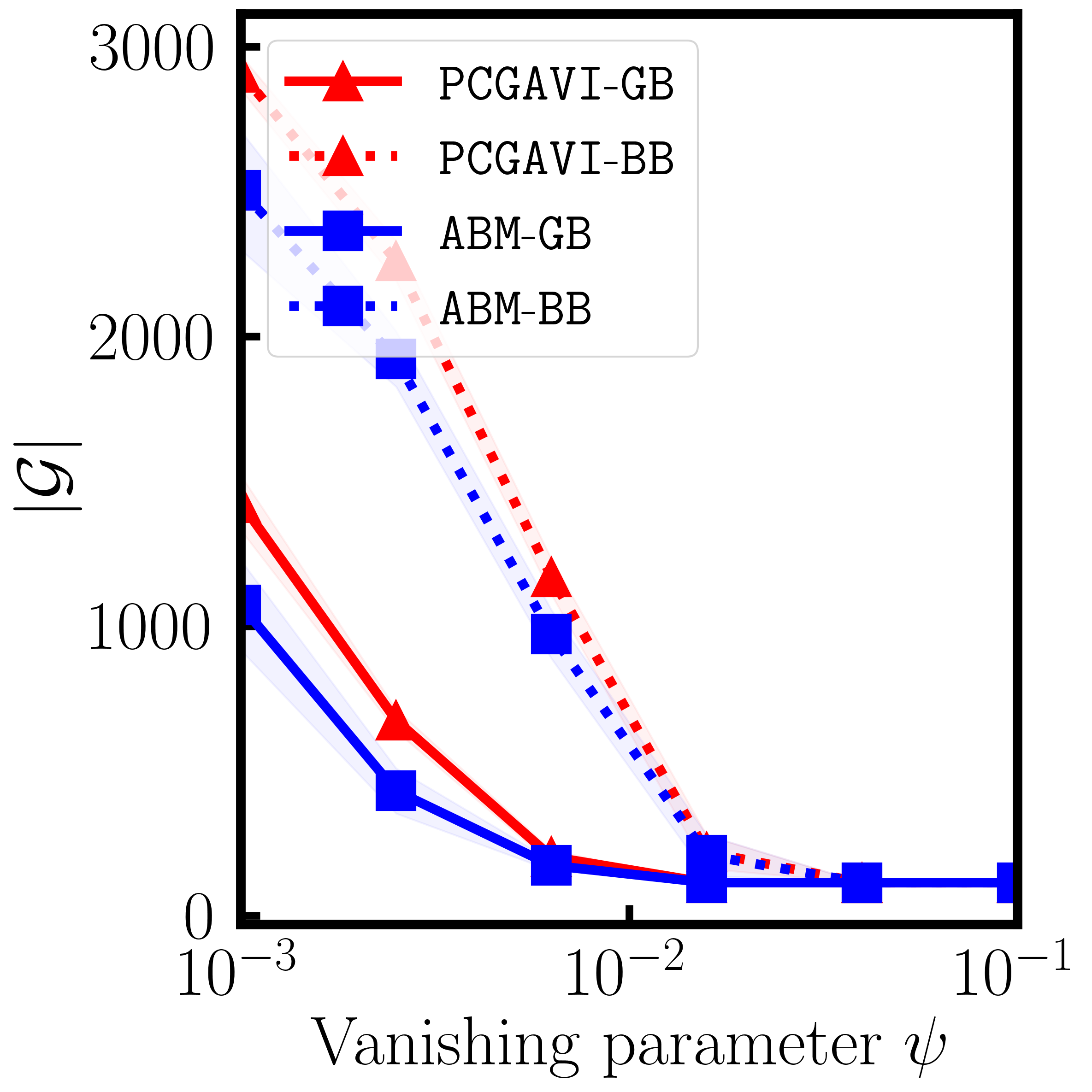}
        \caption{spam}
        \label{fig:border_spam}
    \end{subfigure} 
\end{tabular}
\caption{Comparison of the number of generators constructed with \pcg{}\avi{} and \abm{} for the border-\eqref{eq:gb} and  the border-\eqref{eq:bb}, averaged over ten random runs with shaded standard deviations. Running algorithms with with the border-\eqref{eq:bb} often leads to the construction of more generators than with the border-\eqref{eq:gb}. See Section~\ref{sec:border_types_G} for details on the setup.
}\label{fig:border}
\end{figure}

\section{Numerical Experiments}\label{section:numerical_experiments}

In this section, we compare the performance of \oavi{} as a preprocessing technique for a subsequently applied linear kernel \svm{} to related approaches and determine the influence of the border type on generator-constructing algorithms.

\subsection{General Setup}\label{sec:general_setup}

In this section, we present the information relevant to all our numerical experiments.

\subsubsection{Implementation}
The numerical experiments are implemented in \textsc{Python} and performed on an Nvidia GeForce RTX 3080 GPU with 10GB RAM and an Intel Core i7 11700K 8x CPU at 3.60GHz with 64 GB RAM.
Our code is publicly available on 
\href{https://github.com/ZIB-IOL/cgavi/releases/tag/v2.0.0}{GitHub}.
For all generator-constructing algorithms, we use the definition of $\psi$-approximately vanishing in Definition~\ref{definition:approximately_vanishing_ideal}.

We implement \emph{accelerated gradient descent} (\agd{}) \citep{nesterov1983method} and \pcg{} as \cvxoracle{}s in \oavi{} and refer to the resulting algorithms as \pcg{}\avi{}, and \agd{}\avi{}, respectively. The solvers are run for up to 10,000 iterations. \pcg{} is run up to accuracy $\eps = 0.001 \cdot \psi$. For \pcg{}, we replace \eqref{eq:cop} with \eqref{eq:ccop} with $\tau = 1000$ and the norm $\|\cdot\|_1$. We terminate \pcg{} when less than $0.000001 \cdot \psi$ progress is made in absolute difference between function values, when the coefficient vector of a generator is constructed, or if we have a guarantee that no coefficient vector of a generator can be constructed.
We terminate \agd{} when less than $0.000001 \cdot \psi$ progress is made in absolute difference between function values for 20 iterations in a row or the coefficient vector of a generator is constructed.
Unless noted otherwise, we run \oavi{} with the border-\eqref{eq:gb} as in Definition~\ref{definition:border}.

We implement \abm{} as presented in \citet{limbeck2013computation} with the modification that instead of applying the \svd{} to the matrix corresponding to $A = \cO(X)$ in \oavi{}, we apply the \svd{} to $A^\intercal A$ in case this leads to a faster training time. Unless noted otherwise, we run \abm{} with the border-\eqref{eq:gb} as in Definition~\ref{definition:border}.

We implement \vca{} as presented in \citet{livni2013vanishing} with the modification that instead of applying the \svd{} to the matrix corresponding to $A = \cO(X)$ in \oavi{}, we apply the \svd{} to $A^\intercal A$ in case this leads to a faster training time.

We use a polynomial kernel \svm{} with one-versus-rest approach from the \textsc{scikit-learn} software package \citep{pedregosa2011scikit}. We run the polynomial kernel \svm{} with $\ell_2$-regularization up to tolerance  $10^{-3}$ or for up to 10,000 iterations.

\oavi{}, \abm{}, and \vca{} are used as preprocessing techniques for a subsequently applied linear kernel \svm{} using the machine learning pipeline discussed in Section~\ref{section:the_cgavi_pipeline}. We refer to the combined approaches as  \oavi{}$^*$, \abm{}$^*$, and \vca{}$^*$, respectively. The linear kernel \svm{} is implemented using the \textsc{scikit-learn} software package and run with $\ell_1$-penalized squared hinge loss up to tolerance  $10^{-4}$ or for up to 10,000 iterations.

\subsubsection{Hyperparameters}
\sloppy 
The hyperparameters for \pcg{}\avi{}$^*$, \agd{}\avi{}$^*$, \abm{}$^*$, and \vca{}$^*$ are the vanishing parameter 
$\psi\in \{0.1, 0.05, 0.01, 0.005, 0.001, 0.0005\}$ and the $\ell_1$-regularization coefficient  of the linear kernel \svm{} in $\{0.1, 1, 10\}$. For the polynomial kernel \svm{}, the hyperparameters are the degree of the kernel in $\{1, 2, 3, 4\}$ and the $\ell_2$-regularization coefficient in $\{0.1, 1, 10\}$.

\subsubsection{Data Sets} 

We provide an overview of the data sets in Table~\ref{table:data_sets}. For each data set, we apply min-max feature scaling into the range $[0, 1]$ as a preprocessing step.

\subsection{Experiment: Performance}\label{sec:performance}

\begin{table*}[t]
\centering
\begin{tabular}{|l|l|r|r|}
\hline
Data Set & Full Name &  \# Samples & \# Features  \\
\hline
bank &  banknote authentication & 1,372 & 4 \\
credit &  default of credit cards \citep{yeh2009comparisons} & 30,000 & 22  \\
htru & HTRU2 \citep{lyon2016fifty} &  17,898 & 8  \\
seeds & seeds&  210 & 7 \\
skin & skin \citep{bhatt2010skin} & 245,057  &  3  \\
spam & spambase &  4,601 & 57 \\
\hline
\end{tabular}
\caption{All data sets are binary classification data sets, except for seeds, which is made up of three classes. The data sets are retrieved from the UCI Machine Learning Repository \citep{dua2017uci} and additional references are provided.
}\label{table:data_sets}
\end{table*}
We compare the performance of \pcg{}\avi{}$^*$, \agd{}\avi{}$^*$, \abm{}$^*$, \abm{}$^*$, and polynomial kernel \svm{} on various data sets.

\subsubsection{Setup}\label{sec:setup_performance}
We apply \pcg{}\avi{}$^*$, \agd{}\avi{}$^*$, \abm{}$^*$, \vca{}$^*$, and polynomial kernel \svm{} to data sets bank, credit, htru, seeds, skin, and spam. The hyperparameters are tuned on the training data using threefold cross-validation. We retrain the algorithms on the entire training data using the best hyperparameter combination and compare the classification error on the test set, the hyperparameter optimization time, and the test time, that is, the time to evaluate a method on new data. For generator-constructing approaches, we also compare $|\cG| + |\cO|$, where $|\cG| = \sum_{i} |\cG^i|, |\cO| = \sum_{i} |\cO^i|$, and $(\cG^i, \cO^i)$ is the output of a generator-constructing algorithm applied to samples belonging to class $i$. Moreover, for the generator-constructing approaches, we also compare the sparsity of the feature transformation. The results are averaged over ten random $60 \% / 40 \%$ train/test/partitions.

\subsubsection{Results}
\begin{table*}[t]
\centering
\begin{tabular}{|l|l|r|r|r|r|r|r|}
\hline
\multicolumn{2}{|c|}{\multirow{2}{*}{Algorithms}} & \multicolumn{6}{c|}{Data Sets}  \\ \cline{3-8}
\multicolumn{2}{|c|}{} & {bank} & {credit} & {htru} & {seeds} & {skin} & {spam} \\
\hline
\multirow{5}{*}{\rotatebox[origin=c]{90}{Error Test}}
 & \pcg{}\avi{}$^*$ & $0.51$ & \bm{$17.96$} & $2.11$ & \bm{$3.69$}& $0.26$ & $7.16$ \\
 & \agd{}\avi{}$^*$ & \bm{$0.00$} & $18.14$ & \bm{$2.07$} & $4.76$ & \bm{$0.20$}& \bm{$6.67$}\\
 & \abm{}$^*$ & $0.47$ & $18.36$ & $2.12$ & $5.36$ & $0.43$ & $7.11$ \\
 & \vca{}$^*$ & $0.16$ & $19.85$ & $2.11$ & $5.00$ & $0.24$ & $7.15$ \\
 & \svm{} & \bm{$0.00$} & $18.34$ & $2.08$ & $4.76$ & $2.25$ & $7.13$ \\
 \hline
\multirow{5}{*}{\rotatebox[origin=c]{90}{Time Hyper.}}
& \pcg{}\avi{}$^*$ & $3.6 \times 10^{2}$ & $1.1 \times 10^{3}$ & $3.4 \times 10^{2}$ & $5.9 \times 10^{2}$ & $2.3 \times 10^{2}$ & $2.4 \times 10^{3}$ \\
 & \agd{}\avi{}$^*$ & $2.5 \times 10^{1}$ & $3.0 \times 10^{2}$ & $5.9 \times 10^{1}$ & $4.4 \times 10^{1}$ & $9.9 \times 10^{1}$ & $2.1 \times 10^{2}$ \\
 & \abm{}$^*$ & $3.7 \times 10^{-1}$ & \bm{$1.5 \times 10^{1}$} & $7.7 \times 10^{0}$ & $6.6 \times 10^{-1}$ & $1.8 \times 10^{1}$ & $2.2 \times 10^{1}$ \\
 & \vca{}$^*$ & $7.6 \times 10^{-1}$ & $1.7 \times 10^{1}$ & $4.6 \times 10^{0}$ & $3.8 \times 10^{0}$ & \bm{$1.3 \times 10^{1}$} & $3.4 \times 10^{1}$ \\
 & \svm{} & \bm{$7.0 \times 10^{-2}$} & $8.9 \times 10^{1}$ & \bm{$4.1 \times 10^{0}$} & \bm{$3.0 \times 10^{-2}$} & $7.1 \times 10^{2}$ & \bm{$2.2 \times 10^{0}$} \\
\hline
\multirow{5}{*}{\rotatebox[origin=c]{90}{Time Test}}
& \pcg{}\avi{}$^*$ & $1.5 \times 10^{-3}$ & $3.5 \times 10^{-3}$ & $2.3 \times 10^{-3}$ & $1.2 \times 10^{-3}$ & $8.8 \times 10^{-3}$ & \bm{$1.6 \times 10^{-3}$} \\
 & \agd{}\avi{}$^*$ & $1.4 \times 10^{-3}$ & $6.5 \times 10^{-3}$ & $1.8 \times 10^{-3}$ & $1.5 \times 10^{-3}$ & $8.6 \times 10^{-3}$ & \bm{$1.6 \times 10^{-3}$} \\
 & \abm{}$^*$ & $1.2 \times 10^{-3}$ & \bm{$3. \times 10^{-3}$} & \bm{$1.2 \times 10^{-3}$} & $1.2 \times 10^{-3}$ & $7.2 \times 10^{-3}$ & $1.9 \times 10^{-3}$ \\
 & \vca{}$^*$ & $7.1 \times 10^{-4}$ & $3.1 \times 10^{-3}$ & $1.3 \times 10^{-3}$ & $1.4 \times 10^{-3}$ & \bm{$6.8 \times 10^{-3}$} & $7.6 \times 10^{-3}$ \\
 & \svm{} & \bm{$2.9 \times 10^{-4}$} & $1.4 \times 10^{0}$ & $5.6 \times 10^{-2}$ & \bm{$1.8 \times 10^{-4}$} & $1.1 \times 10^{1}$ & $2. \times 10^{-2}$ \\
\hline
\hline
\multirow{4}{*}{\rotatebox[origin=c]{90}{$|\cG| + |\cO|$}}
 & \pcg{}\avi{}$^*$ & $36.80$ & $67.30$ & $60.00$ & $43.50$ & $32.30$ & \bm{$152.10$} \\
 & \agd{}\avi{}$^*$ & $35.10$ & $213.70$ & $42.10$ & $60.10$ & $27.00$ & $168.40$ \\
 & \abm{}$^*$ & $28.80$ & $51.40$ & \bm{$18.70$} & \bm{$37.90$} & $19.30$ & $259.50$ \\
 & \vca{}$^*$ & \bm{$23.80$} & \bm{$49.80$} & $19.00$ & $103.10$ & \bm{$9.00$} & $1766.40$ \\
 \hline
 \multirow{4}{*}{\rotatebox[origin=c]{90}{\eqref{eq:sparsity}}}
 & \pcg{}\avi{}$^*$ & \bm{$0.17$} & \bm{$0.52$} & \bm{$0.43$} & \bm{$0.29$} & \bm{$0.28$} & \bm{$0.49$} \\
 & \agd{}\avi{}$^*$ & $0.00$ & $0.00$ & $0.00$ & $0.00$ & $0.00$ & $0.01$ \\
 & \abm{}$^*$ & $0.00$ & $0.00$ & $0.00$ & $0.00$ & $0.00$ & $0.00$ \\
 & \vca{}$^*$ & $0.00$ & $0.00$ & $0.00$ & $0.00$ & $0.00$ & $0.00$ \\
 \hline
\end{tabular}
\caption{Comparison of test set classification error in percent, hyperparameter optimization time in seconds, and test time in seconds. For the generator-constructing approaches, we also compare the magnitude of $|\cG| + |\cO|$ and \eqref{eq:sparsity}.
The results are averaged over ten random $60 \%$/$40 \%$ train/test partitions and the best results in each category are in bold.
}\label{table:results}
\end{table*}

We present the results in Table~\ref{table:results}. 
The classification error on the test set for \oavi{}$^*$ is competitive with other approaches. Indeed, for \agdavi{}$^*$, the classification error on the test set is either equal to or smaller than the classification error of non-\oavi{}-based approaches. \pcg{}\avi{}$^*$ tends to perform slightly worse in terms of classification error than \agdavi{}$^*$ but remains competitive with the other approaches.
The hyperparameter optimization time of \oavi{}$^*$ tends to be slower than that of \abm{}$^*$ and \vca{}$^*$. \pcg{}\avi{}$^*$ is always slower than \agd{}\avi{} by one order of magnitude. The hyperparameter tuning time of the polynomial kernel \svm{} tends to be very competitive on smaller data sets. However, on skin, a data set of 245,057 samples, the \svm{} is slower than all other approaches.
The test times of the generator-based approaches tend to be of similar order of magnitudes. For the smaller data sets bank and seeds, the test time of the \svm{} is faster than that of the generator-based approaches. For larger data sets, however, the test time of the \svm{} is often multiple orders of magnitude greater than that of the generator-based approaches.
For \abm{}$^*$ and \vca{}$^*$, $|\cG|+|\cO|$ is often smaller than for \oavi{}$^*$. On spam, the data set with the most features, $|\cG|+|\cO|$ is significantly smaller for \oavi{}$^*$ than for \abm{}$^*$ and \vca{}$^*$.
Only \pcg{}\avi{}$^*$ constructs sparse feature transformations.

\subsection{Experiment: Performance Comparison of Different Border Types }\label{sec:border_types_performance}

\begin{table*}[t]
\centering
\begin{tabular}{|l|l|r|r|r|r|r|r|}
\hline
\multicolumn{2}{|c|}{\multirow{2}{*}{Algorithms}} & \multicolumn{6}{c|}{Data Sets}  \\ \cline{3-8}
\multicolumn{2}{|c|}{} & {bank} & {credit} & {htru} & {seeds} & {skin} & {spam} \\
\hline
\multirow{6}{*}{\rotatebox[origin=c]{90}{Error Test}}
 & \pcg{}\avi{}-\eqref{eq:gb}$^*$ & $0.51$ & \bm{$17.96$} & $2.11$ & \bm{$3.69$} & $0.26$ & $7.16$ \\
 & \pcg{}\avi{}-\eqref{eq:bb}$^*$ & $0.55$ & $18.02$ & $2.14$ & $4.76$ & \bm{$0.20$} & $7.14$ \\
 & \agd{}\avi{}-\eqref{eq:gb}$^*$ & \bm{$0.00$} & $18.14$ & $2.07$ & $4.76$ & \bm{$0.20$} & $6.67$ \\
 & \agd{}\avi{}-\eqref{eq:bb}$^*$ & \bm{$0.00$} & $18.02$ & \bm{$2.04$} & $4.76$ & \bm{$0.20$} & \bm{$6.51$} \\
 & \abm{}-\eqref{eq:gb}$^*$ & $0.47$ & $18.36$ & $2.12$ & $5.36$ & $0.43$ & $7.11$ \\
 & \abm{}-\eqref{eq:bb}$^*$ & $0.26$ & $18.11$ & $2.08$ & $4.52$ & $0.27$ & $7.16$ \\
 \hline
\multirow{6}{*}{\rotatebox[origin=c]{90}{Time Hyper.}}
 & \pcg{}\avi{}-\eqref{eq:gb}$^*$ & $3.6 \times 10^{2}$ & $1.1 \times 10^{3}$ & $3.4 \times 10^{2}$ & $5.9 \times 10^{2}$ & $2.3 \times 10^{2}$ & $2.4 \times 10^{3}$ \\
 & \pcg{}\avi{}-\eqref{eq:bb}$^*$ & $4.4 \times 10^{2}$ & $1.8 \times 10^{3}$ & $3.7 \times 10^{2}$ & $7.5 \times 10^{2}$ & $2.9 \times 10^{2}$ & $2.5 \times 10^{3}$ \\
 & \agd{}\avi{}-\eqref{eq:gb}$^*$ & $2.5 \times 10^{1}$ & $3.0 \times 10^{2}$ & $5.9 \times 10^{1}$ & $4.4 \times 10^{1}$ & $9.9 \times 10^{1}$ & $2.1 \times 10^{2}$ \\
 & \agd{}\avi{}-\eqref{eq:bb}$^*$ & $3. \times 10^{1}$ & $1.0 \times 10^{3}$ & $1.0 \times 10^{2}$ & $7.9 \times 10^{1}$ & $1.4 \times 10^{2}$ & $3.8 \times 10^{2}$ \\
 & \abm{}-\eqref{eq:gb}$^*$ & \bm{$3.7 \times 10^{-1}$} & \bm{$1.5 \times 10^{1}$} & \bm{$7.7 \times 10^{0}$} & \bm{$6.6 \times 10^{-1}$} & \bm{$1.8 \times 10^{1}$} & \bm{$2.2 \times 10^{1}$} \\
 & \abm{}-\eqref{eq:bb}$^*$ & $4.4 \times 10^{-1}$ & $4.1 \times 10^{1}$ & $8.9 \times 10^{0}$ & $1.0 \times 10^{0}$ & $2.3 \times 10^{1}$ & $4.9 \times 10^{1}$ \\
\hline
\multirow{6}{*}{\rotatebox[origin=c]{90}{Time Test}}
 & \pcg{}\avi{}-\eqref{eq:gb}$^*$ & $1.5 \times 10^{-3}$ & $3.5 \times 10^{-3}$ & $2.3 \times 10^{-3}$ & $1.2 \times 10^{-3}$ & $8.8 \times 10^{-3}$ & $1.6 \times 10^{-3}$ \\
 & \pcg{}\avi{}-\eqref{eq:bb}$^*$ & $1.7 \times 10^{-3}$ & $3.5 \times 10^{-2}$ & $2.6 \times 10^{-3}$ & $1.3 \times 10^{-3}$ & $1.1 \times 10^{-2}$ & $3.5 \times 10^{-3}$ \\
 & \agd{}\avi{}-\eqref{eq:gb}$^*$ & $1.4 \times 10^{-3}$ & $6.5 \times 10^{-3}$ & $1.8 \times 10^{-3}$ & $1.5 \times 10^{-3}$ & $8.6 \times 10^{-3}$ & $1.6 \times 10^{-3}$ \\
 & \agd{}\avi{}-\eqref{eq:bb}$^*$ & $1.6 \times 10^{-3}$ & $2.1 \times 10^{-2}$ & $2.7 \times 10^{-3}$ & $1.2 \times 10^{-3}$ & $8.2 \times 10^{-3}$ & $3. \times 10^{-3}$ \\
 & \abm{}-\eqref{eq:gb}$^*$ & \bm{$1.2 \times 10^{-3}$} & \bm{$3. \times 10^{-3}$} & \bm{$1.2 \times 10^{-3}$} & \bm{$1.2 \times 10^{-3}$} & \bm{$7.2 \times 10^{-3}$} & \bm{$1.9 \times 10^{-3}$} \\
 & \abm{}-\eqref{eq:bb}$^*$ & $1.3 \times 10^{-3}$ & $4.7 \times 10^{-3}$ & $1.7 \times 10^{-3}$ & $1.4 \times 10^{-3}$ & $7.9 \times 10^{-3}$ & $3.9 \times 10^{-3}$ \\
\hline
\hline
\multirow{6}{*}{\rotatebox[origin=c]{90}{$|\cG| + |\cO|$}}
 & \pcg{}\avi{}-\eqref{eq:gb}$^*$ & $36.80$ & $67.30$ & $60.00$ & $43.50$ & $32.30$ & \bm{$152.10$} \\
 & \pcg{}\avi{}-\eqref{eq:bb}$^*$ & $55.10$ & $1372.20$ & $87.60$ & $78.30$ & $45.00$ & $664.00$ \\
 & \agd{}\avi{}-\eqref{eq:gb}$^*$ & $35.10$ & $213.70$ & $42.10$ & $60.10$ & $27.00$ & $168.40$ \\
 & \agd{}\avi{}-\eqref{eq:bb}$^*$ & $50.20$ & $869.20$ & $90.30$ & $76.80$ & $29.00$ & $555.60$ \\
 & \abm{}-\eqref{eq:gb}$^*$ & \bm{$28.80$} & \bm{$51.40$} & \bm{$18.70$} & \bm{$37.90$} & \bm{$19.30$} & $259.50$ \\
 & \abm{}-\eqref{eq:bb}$^*$ & $34.80$ & $132.40$ & $40.90$ & $80.70$ & $24.30$ & $714.90$ \\
 \hline
 \multirow{6}{*}{\rotatebox[origin=c]{90}{\eqref{eq:sparsity}}}
 & \pcg{}\avi{}-\eqref{eq:gb}$^*$ & $0.17$ & $0.52$ & $0.43$ & $0.29$ & $0.28$ & $0.49$ \\
 & \pcg{}\avi{}-\eqref{eq:bb}$^*$ & \bm{$0.27$} & \bm{$0.68$} & \bm{$0.54$} & \bm{$0.32$} & \bm{$0.29$} & \bm{$0.60$} \\
 & \agd{}\avi{}-\eqref{eq:gb}$^*$ & $0.00$ & $0.00$ & $0.00$ & $0.00$ & $0.00$ & $0.01$ \\
 & \agd{}\avi{}-\eqref{eq:bb}$^*$ & $0.00$ & $0.00$ & $0.00$ & $0.00$ & $0.00$ & $0.02$ \\
 & \abm{}-\eqref{eq:gb}$^*$ & $0.00$ & $0.00$ & $0.00$ & $0.00$ & $0.00$ & $0.00$ \\
 & \abm{}-\eqref{eq:bb}$^*$ & $0.00$ & $0.00$ & $0.00$ & $0.00$ & $0.00$ & $0.01$ \\
 \hline
\end{tabular}
\caption{Comparison of test set classification error in percent, hyperparameter optimization time in seconds, test time in seconds, magnitude of $|\cG| + |\cO|$, and \eqref{eq:sparsity}.
The results are averaged over ten random $60 \%$/$40 \%$ train/test partitions and the best results in each category are in bold.
}\label{table:results_border}
\end{table*}

We compare the performance of \pcg{}\avi{}-\eqref{eq:gb}$^*$, \pcg{}\avi{}-\eqref{eq:bb}{}$^*$, \agd{}\avi{}-\eqref{eq:gb}$^*$, \agd{}\avi{}-\eqref{eq:bb}{}$^*$, \abm{}-\eqref{eq:gb}$^*$, and \abm{}-\eqref{eq:bb}{}$^*$ on various data sets, where the suffixes \eqref{eq:bb} and \eqref{eq:gb} indicate the use of the border-\eqref{eq:bb} and the border-\eqref{eq:gb}, respectively.

\subsubsection{Setup}

We repeat the experiment presented in Section~\ref{sec:performance} using the setup from Section~\ref{sec:setup_performance} but for algorithms \pcg{}\avi{}-\eqref{eq:gb}$^*$, \pcg{}\avi{}-\eqref{eq:bb}{}$^*$, \agd{}\avi{}-\eqref{eq:gb}$^*$, \agd{}\avi{}-\eqref{eq:bb}{}$^*$, \abm{}-\eqref{eq:gb}$^*$, and \abm{}-\eqref{eq:bb}{}$^*$.

\subsubsection{Results}

The results are presented in Table~\ref{table:results_border}. Note that the results of \pcg{}\avi{}$^*$, \agd{}\avi{}$^*$, and \abm{}$^*$ in Table~\ref{table:results} correspond to the results of \pcg{}\avi{}-\eqref{eq:gb}$^*$, \agd{}\avi{}-\eqref{eq:gb}$^*$, and \abm{}-\eqref{eq:gb}$^*$ in Table~\ref{table:results_border} but are restated for convenience. 
Running the algorithms with the border-\eqref{eq:bb} leads to increased size of the output, $|\cG| + |\cO|$, than with the border-\eqref{eq:gb}. The increased size of the output often leads to longer hyperparameter optimization and test times when algorithms are run with the border-\eqref{eq:bb} than with the border-\eqref{eq:gb}. For \pcg{}\avi{}$^*$, it is not clear whether using the border-\eqref{eq:bb} leads to better classification error on the test set than the border-\eqref{eq:gb}. For \agd{}\avi{}$^*$ and \abm{}$^*$, the border-\eqref{eq:bb} tends to facilitate slightly better classification error on the test set than the border-\eqref{eq:gb}. Finally, for \pcg{}\avi{}$^*$, the \eqref{eq:bb}-variant creates a sparser feature transformation than the \eqref{eq:gb}-variant.

\subsection{Experiment: Output-Size Comparison for Different Border Types}\label{sec:border_types_G}

In this section, we compare the number of constructed generators for \pcg{}\avi{}-\eqref{eq:gb}, \pcg{}\avi{}-\eqref{eq:bb}, \abm{}-\eqref{eq:gb}, and \abm{}-\eqref{eq:bb} for varying vanishing parameters $\psi > 0$ on different data sets.

\subsubsection{Setup}
We apply \pcg{}\avi{}-\eqref{eq:gb}, \pcg{}\avi{}-\eqref{eq:bb}, \abm{}-\eqref{eq:gb}, and \abm{}-\eqref{eq:bb} for different values of $\psi$ to the data sets credit, htru, skin, and spam. We plot the number of constructed generators $|\cG| = \sum_i |\cG^i|$, where $\cG^i$ is the set of generators constructed for class $i$. The results are averaged over ten random runs and standard deviations are shaded.

\subsubsection{Results}
The results are presented in Figure~\ref{fig:border}. For all data sets and for both \pcg{}\avi{} and \abm{}, we observe that running the algorithms with the border-\eqref{eq:gb} leads to the construction of sometimes significantly fewer generators than with the border-\eqref{eq:bb}.

\section{Discussion}\label{section: discussion}

We introduced a new algorithm for the construction of generators of the approximate vanishing ideal, \oavi{}, a framework that captures a theoretically well-motivated variant, \pcg{}\avi{}.
Unlike other vanishing ideal algorithms, \pcg{}\avi{} constructs generators that are $\tau$-bounded in $\ell_1$-norm. As a direct consequence, the algorithm admits three learning guarantees. First, we showed that when a generator $g\in \cG$ attains a small value on a point $\xx\in X \subseteq [-1, 1]^n$ in the training sample, then $g$ also attains small values for samples $\yy\in[-1, 1]^n$, whose distance to $\xx$ is small in $\ell_\infty$-norm. Second, under mild assumptions, we showed that generators constructed by \pcg{}\avi{} are guaranteed to not only vanish approximately on in-sample but also on out-sample data. Third, under mild assumptions, we showed that the combined approach of constructing generators with \pcg{}\avi{} for a subsequently applied linear kernel \svm{} admits a margin bound, similar to that of the \svm{}.
Since other generator-constructing algorithms cannot guarantee that the coefficient vectors of constructed generators are $\tau$-bounded in $\ell_1$-norm, \pcg{}\avi{} is the only vanishing ideal algorithm to which these learning guarantees apply. The question remains open whether similar learning guarantees can be derived for related methods. Especially \abm{}, which guarantees that the coefficient vectors of generators have $\ell_2$-norm equal to $1$ is a promising candidate for alternative learning guarantees. An added benefit of \pcg{}\avi{} is that the coefficient vectors of constructed generators tend to be sparse. The learning guarantees and sparsity of output make \pcg{}\avi{} a theoretically better supported alternative to \abm{}, \avi{}, \bbabm{}, and \vca{}. Numerical experiments indicate that the only drawback of \pcg{}\avi{} is the increased time required for tuning hyperparameters. Given the strong empirical performance and theoretical guarantees of \pcg{}\avi{}, we believe the algorithm to be a compelling alternative to related generator-constructing methods. 
\subsubsection*{Acknowledgements}
This research was partially funded by the Deutsche Forschungsgemeinschaft (DFG, German Research Foundation) under Germany´s Excellence Strategy – The Berlin Mathematics Research Center MATH$^+$ (EXC-2046/1, project ID 390685689, BMS Stipend). We thank Hiroshi Kera for pointing out the connection between \oavi{} and \bbabm{}.

\bibliographystyle{apalike}
\bibliography{bibliography}

\end{document}

%% file: prologue.tex
\usepackage[utf8]{inputenc} 
\usepackage[T1]{fontenc}    
\usepackage{url}            
\usepackage{booktabs}       
\usepackage{amsfonts}       
\usepackage{microtype}      

\usepackage{nicefrac} 

\usepackage{amsmath}
\usepackage{amsthm}
\usepackage{amssymb}
\usepackage{thmtools,thm-restate}
\usepackage{thm-restate}
\usepackage{paralist}
\usepackage{comment}
\usepackage{dsfont}

\usepackage{bbm}
\usepackage{float}
\usepackage{balance}
\usepackage{stmaryrd}
\usepackage{multicol}
\usepackage[export]{adjustbox}
\usepackage{multirow}
\usepackage{tabularx}
\usepackage{mathtools}
\usepackage{mdframed}

\usepackage{subcaption}
\usepackage{enumitem}
\setitemize{noitemsep,topsep=0pt,parsep=0pt,partopsep=0pt}

\usepackage{chemmacros}

\usepackage[ruled, linesnumbered]{algorithm2e} %
\usepackage{xcolor} %
\usepackage{tikz}
\usetikzlibrary{fit,calc}

\colorlet{pink}{red!40}
\colorlet{lightblue}{blue!30}
\colorlet{lightgreen}{green!30}

\DontPrintSemicolon

\usepackage{jmlr2e}
\usepackage[margin=1in]{geometry}

\usepackage{natbib}
\setcitestyle{authoryear,round,citesep={;},aysep={,},yysep={;}}
\renewcommand{\cite}[1]{\citep{#1}}

\usepackage[vvarbb]{newtxmath}

\def\agd{{\textnormal{\texttt{AGD}}}}
\def\abm{{\textnormal{\texttt{ABM}}}}
\def\bbabm{{\textnormal{\texttt{BB-ABM}}}}

\def\agdavi{{\textnormal{\texttt{AGDAVI}}}}
\def\avi{{\textnormal{\texttt{AVI}}}}

\def\cg{{\textnormal{\texttt{CG}}}}

\def\oavi{{\textnormal{\texttt{OAVI}}}}
\def\vca{{\textnormal{\texttt{VCA}}}}

\def\svd{{\textnormal{\texttt{SVD}}}}
\def\svm{{\textnormal{\texttt{SVM}}}}
\def\cvxoracle{{\textnormal{\texttt{ORACLE}}}}

\def\pcg{{\textnormal{\texttt{PCG}}}}

\def\grba{{\textnormal{\texttt{GB}}}}
\def\boba{{\textnormal{\texttt{BB}}}}

\def\bb{{\mathbf{b}}}
\def\cc{{\mathbf{c}}}
\def\dd{{\mathbf{d}}}

\def\ss{{\mathbf{s}}}

\def\vv{{\mathbf{v}}}
\def\ww{{\mathbf{w}}}
\def\xx{{\mathbf{x}}}
\def\yy{{\mathbf{y}}}
\def\zz{{\mathbf{z}}}

\def\ssig{{\boldsymbol{\sigma}}}

\newcommand{\E}{\mathbb{E}}

\newcommand{\N}{\mathbb{N}}
\renewcommand{\P}{\mathbb{P}}

\newcommand{\R}{\mathbb{R}}

\newcommand\cA{{\ensuremath{\mathcal{A}}}\xspace}

\newcommand\cC{{\ensuremath{\mathcal{C}}}\xspace}
\newcommand\cD{{\ensuremath{\mathcal{D}}}\xspace}

\newcommand\cG{{\ensuremath{\mathcal{G}}}\xspace}
\newcommand\cH{{\ensuremath{\mathcal{H}}}\xspace}
\newcommand\cI{{\ensuremath{\mathcal{I}}}\xspace}

\newcommand\cO{{\ensuremath{\mathcal{O}}}\xspace}
\newcommand\cP{{\ensuremath{\mathcal{P}}}\xspace}

\newcommand\cR{{\ensuremath{\mathcal{R}}}\xspace}

\newcommand\cT{{\ensuremath{\mathcal{T}}}\xspace}
\newcommand\cU{{\ensuremath{\mathcal{U}}}\xspace}
\newcommand\cV{{\ensuremath{\mathcal{V}}}\xspace}

\newcommand\cX{{\ensuremath{\mathcal{X}}}\xspace}
\newcommand\cY{{\ensuremath{\mathcal{Y}}}\xspace}

\newcommand{\eps}{\epsilon}
\newcommand{\sig}{\sigma}

\DeclareMathOperator{\sgn}{sign}
\DeclareMathOperator{\mathspan}{span}

\DeclareMathOperator{\mse}{mse}

\DeclareMathOperator{\conv}{conv}

\DeclareMathOperator{\argmin}{argmin}
\DeclareMathOperator{\argmax}{argmax}

\DeclareMathOperator{\spar}{spar}
\DeclareMathOperator{\lt}{lt}
\DeclareMathOperator{\ltc}{ltc}
\newcommand{\zeros}{\ensuremath{\mathbf{0}}}

\newcommand{\oneterm}{\ensuremath{\mathbb{1}}}

\newcommand{\abs}[1]{\lvert #1 \rvert}

\newcommand*{\vsepfbox}[1]{%
  \begingroup
    \sbox0{\fbox{#1}}%
    \setlength{\fboxrule}{0pt}%
    \mbox{\kern-\fboxsep\fbox{\unhbox0}\kern-\fboxsep}%
  \endgroup
}

\theoremstyle{plain} \numberwithin{equation}{section}
\newtheorem{theorem}{Theorem}[section]
\numberwithin{theorem}{section}
\newtheorem{lemma}[theorem]{Lemma}

\newtheorem{proposition}[theorem]{Proposition}

\newtheorem{problem}[theorem]{Problem}
\theoremstyle{definition}
\newtheorem{definition}[theorem]{Definition}

\theoremstyle{plain}

\graphicspath{{imgs/}}
\makeatletter
\def\mathcolor#1#{\@mathcolor{#1}}
\def\@mathcolor#1#2#3{%
  \protect\leavevmode
  \begingroup
    \color#1{#2}#3%
  \endgroup
}
\makeatother
\newcommand{\hrulealg}[0]{\vspace{1mm} \hrule \vspace{1mm}}